\documentclass[sigconf]{acmart}

\copyrightyear{2024}
\acmYear{2024}
\setcopyright{rightsretained}
\acmConference[E-Energy '24]{The 15th ACM International Conference on Future and Sustainable Energy Systems}{June 4--7, 2024}{Singapore, Singapore} 
\acmBooktitle{The 15th ACM International Conference on Future and Sustainable Energy Systems (E-Energy '24), June 4--7, 2024, Singapore, Singapore}
\acmDOI{10.1145/3632775.3661938} 
\acmISBN{979-8-4007-0480-2/24/06}

\usepackage[utf8]{inputenc} % allow utf-8 input
\usepackage[T1]{fontenc}    % use 8-bit T1 fonts
\usepackage{hyperref}       % hyperlinks
\usepackage{url}            % simple URL typesetting
\usepackage{booktabs}       % professional-quality tables
\usepackage{amsfonts}       % blackboard math symbols
\usepackage{nicefrac}       % compact symbols for 1/2, etc.
\usepackage{microtype}      % microtypography
\usepackage{xcolor}         % colors
\usepackage{multirow}
\usepackage{color}
\usepackage{amsthm}
\usepackage{amsmath}
\usepackage{xspace}       
\usepackage{subfig}
\usepackage[ruled,vlined, linesnumbered]{algorithm2e}
\usepackage{algorithmic}

\newtheorem{lemma}{Lemma}
\newtheorem{theorem}{Theorem}

\newcommand{\distancealg}{\texttt{GLB-Nearest}\xspace}
\newcommand{\costalg}{\texttt{GLB-Energy}\xspace}
\newcommand{\carbonalg}{\texttt{GLB-Carbon}\xspace}
\newcommand{\wateralg}{\texttt{GLB-Water}\xspace}
\newcommand{\costcarbonalg}{\texttt{GLB-C2}\xspace}
\newcommand{\allalg}{\texttt{GLB-All}\xspace}
\newcommand{\ouralg}{\texttt{eGLB}\xspace}

\newcommand{\ouralgonline}{\texttt{eGLB}\xspace}

\newcommand{\ouralgoffline}{\texttt{eGLB-Off}\xspace}
\newcommand{\ouralgmpc}{\texttt{eGLB-MPC}\xspace}

\newcommand{\cX}{\mathcal{X}}
\newcommand{\cZ}{\mathcal{Z}}
\newcommand{\cN}{\mathcal{N}}
\newcommand{\cC}{\mathcal{C}}
\newcommand{\cW}{\mathcal{W}}
\newcommand{\cH}{\mathcal{H}}
\newcommand{\RR}{\mathbb{R}}

\begin{document}

\title{Towards Environmentally Equitable AI \\via Geographical Load Balancing}

\author{Pengfei Li}
\email{pli081@ucr.edu}
\affiliation{%
  \institution{University of California, Riverside}
  \country{}
  \city{}
}

\author{Jianyi Yang}
\email{jyang239@ucr.edu}
\affiliation{%
	\institution{University of California, Riverside}
 \country{}
 \city{}
}

\author{Adam Wierman}
\email{adamw@caltech.edu}
\affiliation{%
\institution{California Institute of Technology}
 \country{}
 \city{}
}

\author{Shaolei Ren}
\email{shaolei@ucr.edu}
\affiliation{%
	\institution{University of California, Riverside}
 \country{}
 \city{}
}

\begin{abstract}
Fueled by the soaring popularity of 
foundation models, the accelerated growth of 
artificial intelligence (AI) models' enormous environmental
footprint has come
under increased scrutiny. 
While many approaches have been proposed
to make AI more energy-efficient and environmentally friendly,
 environmental inequity --- the fact that AI's environmental footprint can be disproportionately higher in certain regions than in
others --- has emerged, raising social-ecological 
justice concerns.
This paper takes a first step toward addressing AI's
environmental inequity by 
fairly balancing its regional environmental impact. 
Concretely,
we focus on the carbon and water footprints of AI model 
inference and propose equity-aware 
geographical load balancing (\ouralg) to explicitly 
minimize AI's highest environmental cost across all
the regions.
The consideration of environmental
equity creates substantial algorithmic challenges as the optimal GLB decisions
require complete offline information that is lacking practice.
To address the challenges, we introduce auxiliary variables
and optimize GLB decisions online based on dual
mirror descent.
In addition to analyzing the performance of \ouralg theoretically,
we run trace-based empirical simulations by considering a set of geographically distributed data centers that serve inference requests for a large language AI model. The results demonstrate
that existing GLB approaches may
amplify environmental inequity while 
\ouralg can significantly reduce
the regional disparity in terms of carbon and water footprints. 
%\emph{Source code}: 
%\url{https://github.com/Ren-Research/Environmentally-Equitable-AI}
\end{abstract}

\maketitle

\setcounter{section}{0}

\section{Introduction}

The success of artificial intelligence (AI) relies heavily on
 computationally intensive calculations
 to learn useful information from data during training and
provide insightful predictions during inference. As such, AI models,
 especially large generative models like
GPT-3 \cite{ML_GPT3_Energy_Others_NIPS_2020_NEURIPS2020_1457c0d6},
are typically trained
on large clusters of power-hungry servers that may each have multiple graphic processing
units (GPUs) and
are housed in warehouse-scale data centers.
Moreover, for inference, AI models are often deployed in geographically distributed data centers to serve users
with low transmission latency.

Consequently, the exponentially growing demand for AI
has created an enormous appetite for energy
as well as a negative impact on the environment
\cite{ML_CarbonFoorptint_Google_JeffDean_Journal_2022_9810097,GreenAI_Washington_ACM_2020_10.1145/3381831,GreenAI_EnergyPolicy_NLP_UMass_ACL_2019_strubell-etal-2019-energy,GreenAI_ReportingEnergyCarboon_Stanford_JMLR_2020_10.5555/3455716.3455964,ML_Carbon_Bloom_176B_Sasha_Luccioni_arXiv_2022_luccioni2022estimating,ML_GPT3_Energy_Others_NIPS_2020_NEURIPS2020_1457c0d6}.
For example, putting aside the environmental toll
of chip manufacturing (e.g., raw material
extraction and toxic chemicals) \cite{Carbon_JunkyardComputing_Smartphone_UCSD_ASPLOS_2023_10.1145/3575693.3575710,OECD_MeasuringEnvronmentalImpacts_AI_2022_/content/paper/7babf571-en,Carbon_LCA_ACT_ArchitecturalModeling_HarvardFacebook_ISCA_2022_10.1145/3470496.3527408} and the noise pollution of running AI servers \cite{DataCenter_EnvironmentalImpacts_CarbonWaterNoise_MIT_News_2022},
training a large language model like GPT-3 and LaMDA can easily
consume hundreds of megawatt-hour of electricity,
generate many tonnes of carbon emissions,
and evaporate hundreds of thousands of liters of clean freshwater
for cooling \cite{ML_Carbon_Bloom_176B_Sasha_Luccioni_arXiv_2022_luccioni2022estimating,ML_LaMDA_Lanugage_Google_arXiv_2022_thoppilan2022lamda,Shaolei_Water_AI_Thirsty_arXiv_2023_li2023making}. Crucially, in addition to their impacts on the global climate, AI's environmental footprint also has
significant local and regional impacts. Elevated carbon emissions
have localized social costs \cite{Carbon_LocalCarbonPolicy_Economic_NBER_2022_NBERw30027}
and may increase local ozone,  particulate matter, and premature 
mortality \cite{Carbon_LocalImpact_Stanford_2010_doi:10.1021/es903018m};
electricity generation, especially when burning fuels,
produces \emph{local} air pollutants, discharges
pollution such as thermal pollution into water bodies,
and generates solid wastes (possibly including hazardous wastes) \cite{US_EPA_Electricity_LocalImpact};
and staggering water consumption, both directly for on-site cooling
and indirectly for off-site electricity generation, can further stress the already-limited local freshwater resources
and even worsen extended megadroughts in regions like Arizona \cite{Shaolei_Water_AI_Thirsty_arXiv_2023_li2023making,Water_ElectricityGeneration_US_2003_torcellini2003consumptive}.

Fueled by the soaring popularity of large language
and foundation models, the accelerated growth of AI's environmental
footprint has come
under increased scrutiny recently \cite{AI_SustainableAI_Regulation_PrivacyLawScholarsConf_2023,Carbon_SustainbleAI_CaroleWu_MLSys_2022_wu2022sustainable}. To make AI more energy-efficient
and environmentally friendly, research studies have pursued a variety of approaches, including
 computationally efficient
training and inference \cite{DNN_DeepSpeed_YuxiongHe_MSR_ICML_2022_rajbhandari2022deepspeed,ML_FrugalGPT_Cost_JamesZou_Stanford_arXiv_2023_chen2023frugalgpt},
energy-efficient GPU and accelerator designs \cite{DNN_AutoDNNChip_FPGA_ASIC_YingyanLin_Rice_DemingChen_UIUC_FPGA_2020_10.1145/3373087.3375306,DNN_NAS_AcceleratorAware_AutoML_SimulationPredictor_gupta2020accelerator,ML_CarbonFoorptint_Google_JeffDean_Journal_2022_9810097},
carbon-aware task scheduling \cite{GreenAI_ReportingEnergyCarboon_Stanford_JMLR_2020_10.5555/3455716.3455964,Carbon_SustainbleAI_CaroleWu_MLSys_2022_wu2022sustainable},
green cloud infrastructures \cite{Carbon_MetricsSustainability_AnshulGandhi_StonyBrook_HotCarbon_2022,Carbon_SustainableClouds_VirtualizingEnergy_DavidIrwin_AdamWierman_SoCC_2021_10.1145/3472883.3487009,Carbon_CarbonExplorer_HolisticGreenDataCenter_CaroleWu_BenjaminLee_ASPLOS_2023_10.1145/3575693.3575754},
sustainable AI policy recommendations \cite{OECD_MeasuringEnvronmentalImpacts_AI_2022_/content/paper/7babf571-en,AI_SustainableAI_Regulation_PrivacyLawScholarsConf_2023}, among others.
As supply-side solutions, data center operators have also increasingly
adopted carbon-free energy such as solar and wind power, (partially) powering AI servers and lowering
carbon emissions \cite{Carbon_SustainbleAI_CaroleWu_MLSys_2022_wu2022sustainable,Google_SustainabilityReport_2022,Facebook_SustainabilityReport_2021}. Additionally, to reduce
on-site water consumption and mitigate
the stress on already-limited freshwater resources, climate-conscious cooling
system designs (e.g., using air-side economizers
if the climate condition permits) have recently seen an uptick in the data center industry \cite{Google_Water_Commitments_website,Facebook_Water_2023_meta}.

While existing efforts are encouraging, a worrisome outcome --- environmental inequity --- has unfortunately emerged. 
That is, minimizing the total environmental cost of AI across multiple regions
does not necessarily mean each region is treated equitably. In fact, AI's environmental footprint is often disproportionately higher in certain regions than in others, potentially exacerbating other unintended social-ecological consequences \cite{Justice_DoE_website}.
 For example, a data center's on-site cooling water usage effectiveness (WUE, the ratio
of water consumption to IT energy consumption) highly depends
on the outside temperature \cite{Shaolei_Water_SpatioTemporal_GLB_TCC_2018_7420641} ---
while it can stay well below 1.0 L/kWh for data centers located in cooler climates,
the monthly average WUE can be as high
as 9.0 L/kWh in the summer in drought-stricken Arizona
\cite{Water_DataCenterEnergy_Tradeoff_Arizona_Real_Measurement_WUE_Monthly_2022_KARIMI2022106194}. 
Likewise, there exists a significant regional difference
in terms of the carbon efficiency
--- as of 2020, only 4\% of the energy
for Google's data
center in Singapore is carbon-free, whereas
this number goes up to 94\% in Finland \cite{ML_CarbonFoorptint_Google_JeffDean_Journal_2022_9810097}, creating
a 23$\times$ disparity.
Thus, as a result of such regional differences, 
certain data center locations are severely disadvantaged
and more negatively impacted by the environmental toll of AI.
Further compounded by enduring socioeconomic disparities and even potentially
amplified by existing data center scheduling algorithms,
environmental inequity of AI can pose critical business risks
 and hence needs to be properly reconciled.

Indeed, addressing
its environmental inequity is  increasingly important and becoming integral
to responsible AI and computing 
\cite{Justice_AI_DataCenterConstructionLocation_Policy_UMich_Report_2022,Justice_AINowInstitute_CallforEnvironmentJusticeAI_2021}. 
For example, in the first-ever
global agreement to ensure
healthy development of AI,
the United Nations Educational, Scientific and Cultural Organization (UNESCO) recommends that ``AI should not be used'' if it creates
``disproportionate negative impacts on
the environment'' \cite{Justice_Policy_Ethical_AI_Recommendation_UNESCO_2022}.
The AI Now Institute even compares the uneven regional distribution of AI's environmental costs to ``historical practices of settler colonialism and racial capitalism'' in its 2023 Landscape report \cite{Justice_AINowInstitute_ConfrontingTechPower_2023}.
 Among all the environmental-related
topics, Meta ranks environmental \emph{justice} as
the most critical one with the greatest impact
on its business risks and opportunities \cite{Facebook_SustainabilityReport_2021}.
More recently, studies have also emerged to suggest
new regulations pertinent to AI's growing environmental
footprint \cite{AI_SustainableAI_Regulation_PrivacyLawScholarsConf_2023},
and holistic assessment of AI as social-ecological-technological systems
using available tools from environmental justice \cite{Justice_Policy_Discussion_Dobbe_TUDelft_FAccT_2023,Justice_LanguageModel_Dangers_PolicyDiscussion_TimnitGebru_FAccT_2021_10.1145/3442188.3445922}.

In this paper, we take a first step to address the emerging 
environmental inequity of AI by balancing its negative environmental impact
across geographically distributed data centers. More concretely,
we focus on the carbon and water footprints of AI model
inference and dynamically schedule users' inference requests (also referred
to as \emph{workloads} in this paper) using equity-aware
geographical load balancing (GLB) to fairly distribute
AI's environmental cost to each region.
To mitigate environmental inequity, our key novelty is to augment
the traditional cost-saving objective by explicitly 
including minimization of the most significant 
negative environmental impacts among all the data centers. 

Nonetheless, the consideration of environmental equity
in GLB decisions creates substantial algorithmic challenges. Specifically,
due to their dependency on the long-term carbon and water footprints,
the equity-related costs
couple all the GLB decisions over $T$ time slots (see \eqref{eqn:objective}--\eqref{eqn:constraint_gateway}).
This means that the optimal GLB decisions 
require all the offline information (e.g., future workload arrivals and
water efficiency) in advance, while we must make GLB decisions online without
knowing all the future.
To address this challenge, we propose a new online equity-aware GLB algorithm,
called \ouralgonline, which leverages online information
to optimize the GLB decisions based on dual mirror descent.
We also bound the cost performance of \ouralgonline compared
to the offline optimal equity-aware GLB algorithm (\ouralgoffline).

To empirically evaluate our proposed
equity-aware GLB, we run trace-based simulations by considering a set of 10 geographically distributed data centers that serve inference requests for a large language AI model. Our results demonstrate
that the proposed equity-aware GLB can significantly reduce
the carbon and water footprints in the most disadvantaged region.
In stark contrast,
 existing carbon- and water-saving  GLB approaches may even
 amplify environmental inequity, showing that minimizing
 the total environmental footprint does not necessarily
 treat each region fairly.

In summary, our work is the first study to advance AI's
environmental equity via GLB, connecting research across
data center scheduling, sustainable AI, and equitable AI.
It highlights the need and great potential of equity-aware GLB
to fairly distribute AI's environmental cost across different
regions for environmental equity.

\section{Problem Formulation}\label{sec:problem_formulation}
%\subsection{Model}

While AI model training is energy-intensive,
the environmental footprint of its inference phase 
is also enormous and can even be several times higher
than the training process \cite{Carbon_GenerativeAI_AndrewChien_Chicago_HotCarbon_2023_10.1145/3604930.3605705}.
As such,
we 
consider  a pre-trained AI model
(e.g., large language model) 
and focus on the inference phase.
The AI model inference service is deployed
over a set $\mathcal{N}=\{1,\cdots,N\}$ of geographically distributed data centers to serve users in different regions.
There are a set $\mathcal{J}=\{1,\cdots,J\}$ of  front-end
traffic gateways
that aggregate users' requests from their respective
surrounding areas and assign the requests to data centers,
which is also referred to as geographical load balancing (GLB) in the literature \cite{Liu:2011:GGL:1993744.1993767,Carbon_GenerativeAI_AndrewChien_Chicago_HotCarbon_2023_10.1145/3604930.3605705}.
The GLB decisions are made in a time-slotted manner
over a total of $T$ time slots. Each time slot
can range from a few minutes to about an hour, depending on how frequently
the decisions are updated. We also interchangeably use
``workloads'' and ``requests'' when referring to users' demand for the AI model inference service. Our model is consistent with those used
in the literature such as \cite{Liu:2011:GGL:1993744.1993767,Google_CarbonAwareComputing_PowerSystems_2023_9770383,Shaolei_Water_SpatioTemporal_GLB_TCC_2018_7420641}.

Each data center houses a cluster of servers (typically each equipped with multiple GPUs) to host AI models for inference. For
the ease of presentation, we assume a homogeneous
AI model on all the servers, 
while the extension
to heterogeneous AI
models with different model sizes is considered in
Appendix~\ref{sec:heterogeneous_AI_models}.
 %The AI model is hosted on
During each time slot, the maximum service capacity for
the AI model inference is $M_i$ for data center $i$.
We use $\lambda_{j,t}$ to denote the total amount of workloads arriving  at gateway $j$ at time $t$, and $x_{ij}(t)\geq0$ to represent
the GLB decision (i.e., the load assigned to data center $i$
from gateway $j$). For the convenience of presentation,
we also use $x(t)=\{x_{i,j}(t)\,|\,i\in\mathcal{N},j\in\mathcal{J}\}$ as the collection of all the GLB decisions
at time $t$.

The total load
assigned to data center $i$ is $\sum_{j\in\mathcal{J}} x_{ij}(t)\leq M_i$ at time $t$, thus resulting in a total server energy consumption of $e_i(x(t))$ which is an increasing function of $\sum_{j\in\mathcal{J}} x_{ij}(t)$.
For example, a common model \cite{Liu:2011:GGL:1993744.1993767} is
%to express $e_i(x(t))$ can be expressed
%as 
$e_i(x(t))=\rho_{i,t} \bar{E}_{i,s}+ \frac{\sum_{j\in\mathcal{J}} x_{ij}(t)}{M_i}\cdot \bar{E}_{i,d}$ where $\bar{E}_{i,s}$
is the server cluster's static/idle energy
even when no workload is processed in data center $i$,
$\bar{E}_{i,d}$ is the cluster's dynamic energy consumed when only
processing workloads,
 $\frac{\sum_{j\in\mathcal{J}} x_{ij}(t)}{M_i}$ is
the cluster-level utilization, and $\frac{\sum_{j\in\mathcal{J}} x_{ij}(t)}{M_i}\leq \rho_{i,t}\leq1$ indicates how well the cluster is right-sized in proportion
to the workloads (i.e., $\rho_{i,t}=\frac{\sum_{j\in\mathcal{J}} x_{ij}(t)}{M_i}$
means the cluster is perfectly sized to the workloads by turning
off unused servers, while $\rho_{i,t}=1$ means the servers are always kept on
regardless of the assigned workloads).

Next, we model the energy cost,  carbon footprint,
and water footprint
 in terms of the GLB decisions.
Here, we explicitly model carbon and water footprints separately,
as they are two complementary
and \emph{non-substitutable} measures for ecological impacts  \cite{Water_DifferentFromCarbon_Ecological_Journal_FANG2014508}.

{\textbf{Energy cost.}}
Suppose that the electricity price
and power usage effectiveness (PUE, which accounts
for non-IT energy consumption such as cooling systems and
power distribution losses) are
$p_{i,t}$ and $\gamma_{i,t}$ for data center $i$
at time $t$, respectively. Then,  the total
energy cost at time $t$ can be written
as
\begin{equation}\label{eqn:energy_cost}
  g_t({x}(t))=%\sum_{i\in\mathcal{N}} p_{i,t}\gamma_{i,t}e_{i,t}=
  \sum_{i\in\mathcal{N}} p_{i,t}\gamma_{i,t}e_{i}(x(t)).
\end{equation}
 %The price $p_{i,t}$ can also capture other applicable
%energy-proportional charges (e.g., water consumption
%$g({x}(t))=\sum_{i=1}^N p_{i,t}\gamma_{i,t}e_{i,t}=\sum_{i=1}^N %p_{i,t}\gamma_{i,t}e_{i}(\sum_{j=1}^J x_{ij}(t))$.
 Note that, if the AI model inference service is run
on virtual machine (VM) instances rented from public cloud providers,
the electricity price $p_{i,t}$ becomes the VM price subject to the VM instance type
and $g_t({x}(t))=
  \sum_{i\in\mathcal{N}} p_{i,t} e_{i}(x(t))$
is the total VM rental cost at time $t$ where $e_{i}(x(t))$ represents
the number of VM instances rented to process
the assigned workloads in location $i$.

{\textbf{Carbon footprint.}}
The carbon footprint of AI model inference is
embedded in the generation of electricity using
carbon-intensive fuels such as coal \cite{Gao:2012:EG:2377677.2377719,Carbon_LCA_ACT_ArchitecturalModeling_HarvardFacebook_ISCA_2022_10.1145/3470496.3527408,ML_Carbon_Bloom_176B_Sasha_Luccioni_arXiv_2022_luccioni2022estimating}.
By denoting
the carbon intensity as $\alpha_{i,t}$ with a unit
of gram/kWh, we have
the following carbon footprint for data center $i$
at time $t$:
\begin{equation}\label{eqn:carbon_footprint}
\begin{split}
c_{i,t}(x(t)) = %\alpha_{i,t}\gamma_{i,t}e_{i,t}=
\alpha_{i,t}\gamma_{i,t}e_{i}(x(t))
\end{split}
\end{equation}
The carbon intensity $\alpha_{i,t}$ can be obtained by querying
the local utility or averaging
the carbon intensity of the grid's fuel mix \cite{Gao:2012:EG:2377677.2377719}.

{\textbf{Water footprint.}} 
In parallel with the carbon footprint,
data centers' staggering water footprint has recently become
a new focused area for sustainability (see, e.g.,
the pledge of ``Water Positive by 2030'' by big techs \cite{Google_Water_Commitments_website,Microsoft_SustainabilityReport_2023}).
To serve AI model inference,
data centers consume
clean freshwater both directly and indirectly \cite{Water_DataCenterFootprint_EnvironmentalResearcHLetters_VT_2021_siddik2021environmental,Google_Water_Commitments_website,Shaolei_WATCH_TCC_2015,Shaolei_Water_AI_Thirsty_arXiv_2023_li2023making}. 
The direct water consumption comes from
the cooling system to keep servers from overheating.
AI servers use either air or closed-loop
liquid to transfer the heat to the facility level
(e.g., the facility cooling water loop
or heat exchanger \cite{Water_NREL_CoolingHybrid_2018}).
Then, to further reject the heat into the outside environment,
data centers
commonly use cooling towers
due to their
energy efficiency and applicability to a wide
range of weather conditions. Nonetheless, a large amount
of water is evaporated into the outside environment (i.e.,
not discharged or returned to the source)
and hence considered ``consumed'' \cite{Google_Water_Commitments_website}.
For example, depending on the outside wet-bulb temperature, a cooling tower typically consume
1$\sim$4 liters of water
(up to 9 liters of water in the summer) for each kWh server energy \cite{Water_DataCenterEnergy_Tradeoff_Arizona_Real_Measurement_WUE_Monthly_2022_KARIMI2022106194}. Importantly, the vast majority
of the cooling water supply is drinking-grade
(e.g., nearly 90\% for Google's U.S. data centers in 2021 \cite{Google_Water_Commitments_website}).
Alternatively, air-side economizers
(i.e., directly using outside air to cool down servers) can
be used to save water if the climate condition is suitable,
but water is still needed when the outside temperature is high and/or
the humidity is low --- Meta's state-of-the-art cooling
systems use an average of 0.26 liters of water for each kWh
server energy across its global data center fleet in 2021 \cite{Facebook_SustainabilityReport_2021}.
%The same as carbon footprint,
 AI systems are also accountable for water consumption
embedded in the
electricity generation process.
For example,
thermal and nuclear power plants require a large
volume of water consumption for cooling, while
hydropower consumes water by expediting evaporation downstream
\cite{Water_Electricity_EWIF_Water_Intensity_WorkingPaper_WorldResourcesInstitute_2020_reig2020guidance,DataCenter_US_EnergyUsageReport_2016_LBNL}.

%By combining both direct and indirect
%water consumption, 
Thus, the total water footprint for data center $i$
at time $t$ is
\begin{equation}\label{eqn:water_footprint}
\begin{split}
w_{i,t}\left(x(t)\right) = %\left[\epsilon_{i,t}+\beta_{i,t}\gamma_{i,t}\right]\cdot %e_{i,t}=
\left[\epsilon_{i,t}+\beta_{i,t}\gamma_{i,t}\right]\cdot e_{i}\left(x(t)\right),
\end{split}
\end{equation}
where $\epsilon_{i,t}$ is the direct water usage effectiveness
(WUE) for on-site cooling, $\beta_{i,t}$
is the indirect WUE for off-site electricity generation,
and $\gamma_{i,t}$ is the PUE. The direct WUE is defined as the ratio
of water consumption to IT server energy consumption \cite{GreenGrid_Whitepaper_WUE},
and hence we do not need to multiply $\gamma_{i,t}$
when calculating the direct water consumption.
In practice, the direct WUE $\epsilon_{i,t}$
heavily depends on the outside temperature
 %,
%and hence can be modeled as a time-varying function in terms of
%the outside weather condition 
\cite{Water_DataCenterEnergy_Tradeoff_Arizona_Real_Measurement_WUE_Monthly_2022_KARIMI2022106194,Shaolei_Water_AI_Thirsty_arXiv_2023_li2023making}.
Like the carbon intensity, the indirect WUE $\beta_{i,t}$
measures the water consumption per kWh electricity generation and
can be calculated  by averaging
over the water intensity of different energy fuels  \cite{Shaolei_Water_SpatioTemporal_GLB_TCC_2018_7420641}.
The monetary price for on-site water consumption
is typically much smaller compared to the energy cost
and can be factored into the price $p_{i,t}$ for modeling purposes.
 %Note that, as a large bulk of the on-site water cost
%is the fixed connection charge based on the maximum water consumption rate,
%the cost for the actual on-site water usage
%is typically much smaller compared to the energy cost. Thus,
%we exclude the actual monetary water cost from our problem formulation.

%\section{Geographical Load Balancing for Environmentally Equitable AI}

\section{Environmentally Equitable GLB}

To make AI environmentally equitable, we propose
a novel online equity-aware GLB algorithm, called \ouralg,
to distribute AI's environmental cost across
different data centers in a fair manner.

\subsection{Objective}\label{sec:problem_basic_setting}
Our goal is not to blindly \emph{equalize} its
regional environmental footprint, which,
as similarly observed in the context of mitigating AI's
algorithmic
unfairness \cite{Fair_MiniMaxGroupFairness_AaronRoth_Amazon_2021_Diana2021_10.1145/3461702.3462523},
may artificially elevate the environmental footprints in those otherwise
advantaged regions  and provide a false sense of fairness.
Instead, we adopt the notion of \emph{minimax} fairness
\cite{Fair_MinimaxPareto_Fairness_ICML_2020_10.5555/3524938.3525565,Fair_MiniMaxGroupFairness_AaronRoth_Amazon_2021_Diana2021_10.1145/3461702.3462523,Fair_MaxminFair_SchedulingWireless_INFOCOM_2002_1019322}
and exploit the power of GLB as a software-based approach
to explicitly minimize AI's environmental impact
on the most disadvantaged region.

Mathematically, we augment
the traditional cost-saving objective by including the
minimization of the greatest
environmental cost among all the data centers. 
By normalizing the energy cost and environmental footprints over $T$, our
equity-aware GLB problem
is
formulated as follows 
\begin{subequations}\label{eqn:objective_offline}
 \begin{gather}\label{eqn:objective}
    \begin{gathered}
      \min_{x(t)} \frac{1}{T} \sum_{t=1}^Tg_t(x(t))+\mu_c\cdot
       \max_{i\in\mathcal{N}}\left[\mathcal{H}_{i,c}\left(\frac{1}{T} \sum_{t=1}^Tc_{i,t}(x(t))\right)\right] \\
    +\mu_w\cdot\max_{i\in\mathcal{N}}\left[\mathcal{H}_{i,w}\left(\frac{1}{T} \sum_{t=1}^Tw_{i,t}\left(x(t)\right)\right)\right],
    \end{gathered}\\
       \label{eqn:constraint_delay_assignment}
      s.t., \;\;\;\;\;\;    x_{i,j}(t)=0, \;\text{ if } B_{i,j}=0,\;\;\;\; \forall\; i\in{\mathcal{N}},
       j\in{\mathcal{J}}, t=1,\cdots,T,
          \\
          \label{eqn:constraint_datacenter_capacity}
      \sum_{j\in\mathcal{J}}x_{i,j}(t)\leq M_i, \;\;\;\; \forall
       \; i\in\mathcal{N},t=1,\cdots,T,\\
       \label{eqn:constraint_gateway}
         \sum_{i\in\mathcal{N}}x_{i,j}(t)= \lambda_{j,t}, \;\;\;\; \forall
       \; j\in\mathcal{J}, t=1,\cdots,T, 
 \end{gather}
\end{subequations}
where the assignment condition $B_{i,j}=0$
indicates that the workloads cannot be assigned from gateway $j$
to data center $i$ (due to, e.g., latency constraints
or  data sovereignty regulations) and hence enforces $x_{i,j}=0$  in \eqref{eqn:constraint_delay_assignment},
the constraint \eqref{eqn:constraint_datacenter_capacity} means
that the total workloads assigned to a data center cannot exceed
its processing capacity, and the constraint \eqref{eqn:constraint_gateway}
requires that all workloads arriving at a gateway be assigned to data centers.
In the optimization objective \eqref{eqn:objective},
the monotonically-increasing convex functions $\mathcal{H}_{i,c}()$ and $\mathcal{H}_{i,w}()$
quantify the environmental impacts of AI on data center $i$ due
to its long-term carbon footprint
and water footprint, respectively, and can be specified
based on the local environment assessment. 
Note that the carbon footprint is also
a good indicator of the amount
of local air/thermal pollution caused by
 our GLB decisions.
For example,
coal-based energy sources are carbon-intensive
and also proportionally create air and thermal pollution for local communities \cite{US_EPA_Electricity_LocalImpact}.

Using $\mathcal{H}_{i,w}\left(\frac{1}{T}\sum_{t=1}^Tw_{i,t}\left(x(t)\right)\right)=\frac{\theta_i}{T}
\cdot\sum_{t=1}^Tw_{i,t}\left(x(t)\right)$
as an illustrative example,
we can set a higher $\theta_i\geq0$
if data center $i$ is located in a severely water-stressed
and drought-prone region.
In line with the principle of proportionality,
the  carbon footprint
$\sum_{t=1}^Tc_{i,t}(x(t))$
in $\mathcal{H}_{i,c}()$
and water footprint $\sum_{t=1}^Tw_{i,t}(x(t))$
in  $\mathcal{H}_{i,w}()$
for data center $i$ can also be normalized
by the maximum processing capacity $M_i$ to achieve
proportional fair distribution of AI's environmental cost.

The two functions $\mathcal{H}_{i,c}()$ and $\mathcal{H}_{i,w}()$
are general enough and can also capture
the effects of additional sustainability practices that data center operators
may adopt (e.g., installing solar for carbon mitigation and
restoring watersheds for local water supply
\cite{Facebook_SustainabilityReport_2021,Google_SustainabilityReport_2022}).
The term $\sum_{t=1}^Tg_t(x(t))$ in \eqref{eqn:objective}
is the total energy cost.
The hyperparameters $\mu_c\geq0$
and $\mu_w\geq0$ indicate the relative importance weights
of carbon footprint equity and water footprint equity, respectively,
and can be flexibly tuned to balance the impact of carbon and water footprints. For example, by setting $\mu_c=0$, we focus solely
on the negative environmental impact of AI's water footprint.
In addition, we can also include into \eqref{eqn:objective}
AI's other environmental impacts such as concerns with the servers' noise pollution if applicable \cite{DataCenter_EnvironmentalImpacts_CarbonWaterNoise_MIT_News_2022}.

Importantly,
the cost terms
$\max_{i\in\mathcal{N}}\left[\mathcal{H}_{i,c}\left(\frac{1}{T}\sum_{t=1}^Tc_{i,t}(x(t))\right)\right]$
and $\max_{i\in\mathcal{N}}\left[\mathcal{H}_{i,w}\left(\sum{1}{T}\sum_{t=1}^Tw_{i,t}\left(x(t)\right)\right)\right]$ improve environmental
equity by explicitly penalizing the greatest environmental impacts
that AI model inference creates on different regions. This is fundamentally
different from the existing sustainable GLB techniques that have predominantly
focused on minimizing the weighted \emph{sum} of energy costs, carbon footprint and/or
water footprint \cite{Liu:2011:GGL:1993744.1993767,Gao:2012:EG:2377677.2377719,Shaolei_Water_SpatioTemporal_GLB_TCC_2018_7420641,Carbon_CarbonExplorer_HolisticGreenDataCenter_CaroleWu_BenjaminLee_ASPLOS_2023_10.1145/3575693.3575754,LeBianchiniNguyenBilgirMartonosi_IGCC_2010}. As shown in our experiments (Section~\ref{sec:experiment}),
minimizing the total environmental footprint does not necessarily
treat each individual region fairly
and can even potentially exacerbate environmental inequity due to aggressive
 exploitation of certain regions.

\subsection{An Online Algorithm}

The addition of two equity-related costs in \eqref{eqn:objective}
explicitly mitigates the greatest long-term environmental costs across all the data centers. Thus, 
they couple all the GLB decisions over $T$ time slots.
Consequently, the optimal GLB decisions require complete offline information
(including future workload arrivals and carbon/water efficiencies)
in advance, which is lacking in practice.
Next, we propose an 
online algorithm, called \ouralg, to solve
\eqref{eqn:objective}--\eqref{eqn:constraint_gateway} and
optimize equity-aware GLB decisions
in an online manner. 

A crucial step in \ouralg is to construct a new optimization problem
that can be solved based on available online information
by removing the dependency of the optimization
objective on future information.
To this end, we first  transform the original problem
\eqref{eqn:objective}--\eqref{eqn:constraint_gateway}
into an equivalent new problem that can be solved using dual mirror descent (DMD).
Specifically, for every time step $t\in [1,T]$, we introduce a set of auxiliary variables $\{ z_{c}(t), z_{w}(t) \}$ and consider
the following new transformed problem: 
\begin{subequations}\label{eqn:objective_dmd}
 \begin{gather}
 \begin{gathered}\label{eqn:objective_dual} \min_{x(t), z_{c}(t), z_{w}(t)}
 \frac{1}{T} \sum_{t=1}^T g_t(x(t)) + \frac{\mu_c}{T}\sum_{t=1}^T\max_{i\in\mathcal{N}} \left[\mathcal{H}_{i,c}({z}_{i,c}(t) )\right]  \\
       +\frac{\mu_w}{T}\sum_{t=1}^T\max_{i\in\mathcal{N}} \left[ \mathcal{H}_{i,w}({z}_{i,w}(t) ) \right] 
 \end{gathered}\\
    s.t.,\;\;\;\;\;\; \text{ constraints } \eqref{eqn:constraint_delay_assignment}\eqref{eqn:constraint_datacenter_capacity}\eqref{eqn:constraint_gateway} \label{eqn:action_single_step_constraint}\\
       \frac{1}{T}\sum_{t=1}^T{z}_{i,c}(t) \geq \frac{1}{T}   \sum_{t=1}^Tc_{i,t}(x(t)), \;\;\forall i \in  \mathcal{N}\label{eqn:constraint_gateway_dual_1}\\
          \frac{1}{T}   \sum_{t=1}^T{z}_{i,w}(t) \geq \frac{1}{T}  \sum_{t=1}^Tw_{i,t}\left(x(t)\right), \;\; \forall i \in  \mathcal{N} \label{eqn:constraint_gateway_dual_2}
\end{gather}
\end{subequations}
where the auxiliary variables $z_c(t)=({z}_{1,c}(t)\cdots,z_{N,c}(t))$ and $z_w(t)=({z}_{1,w}(t)\cdots,z_{N,w}(t))$  are chosen from a fixed feasible set $\cZ_c$ and $\cZ_w$, respectively. 
Here, we set $\mathcal{Z}_c=\{z_c|0\leq z_{i,c}\leq \bar{z}_{i,c}, \forall i=1,\cdots,N\}$ and $\mathcal{Z}_w=\{z_w|0\leq z_{i,w}\leq \bar{z}_{i,w}, \forall i=1,\cdots,N\}$  to guarantee a feasible solution for any $x_t \in \mathcal{X}_t$. Specifically,
we can choose $\bar{z}_{i,c}$ and $\bar{z}_{i,w}$ to be the maximum possible per-time carbon footprint and water footprint
in data center $i$, respectively. 

Next, we prove the equivalence
of the new transformed problem to the original problem.

\begin{lemma}
    The transformed problem \eqref{eqn:objective_dual}--\eqref{eqn:constraint_gateway_dual_2}
    and the original problem
    \eqref{eqn:objective}--\eqref{eqn:constraint_gateway}
    have the same optimal GLB decisions.
\end{lemma}
\begin{proof}
To prove this, we first define the optimal GLB decisions as $x^*_{1:T}
=(x(1)^*,\cdots,x(T)^*)$ for the original problem \eqref{eqn:objective}--\eqref{eqn:constraint_gateway}. 
Then, we can construct a feasible solution $z_{i, c}(t) = \frac{1}{T}\sum_{t=1}^Tc_{i,t}(x^*(t))$ and $z_{i,w}(t) = \frac{1}{T} \sum_{t=1}^Tw_{i,t}\left(x^*(t)\right)$,  $\forall t \in [1,T]$ for the transformed problem \eqref{eqn:objective_dual}--\eqref{eqn:constraint_gateway_dual_2}, which results in an equivalent objective function value as \eqref{eqn:objective} in the original problem.
Therefore, the optimal value of the transformed objective in ~\eqref{eqn:objective_dual} is less than or equal to that in the original problem.  

On the other hand, suppose that there exists another solution,
denoted as $\{x'(t), z_{c}(t)', z_{w}(t)', t\in[1,T]\}$, 
which minimizes the transformed problem and makes the transformed objective in \eqref{eqn:objective_dual} strictly smaller than the original one in \eqref{eqn:objective}. By the convexity assumption of $\cH_{i,w}(\cdot)$ and $\cH_{i,c}(\cdot)$ and Jensen's inequality, we have
\begin{gather}
\max_{i \in \cN}\left[\cH_{i,c}\left(\frac{1}{T}\sum_{t=1}^T{z}'_{i,c}(t)\right)\right] \leq \frac{1}{T}\sum_{t=1}^T\max_{i \in \cN}\left[\cH_{i,c}\left({z}'_{i,c}(t)\right)\right],\\
   \max_{i \in \cN}\left[\cH_{i,w}\left(\frac{1}{T}\sum_{t=1}^T{z}'_{i,w}(t)\right)\right] \leq \frac{1}{T}\sum_{t=1}^T\max_{i \in \cN}\left[\cH_{i,w}\left({z}'_{i,w}(t)\right)\right].
\end{gather}
Based on the monotonically increasing assumption
on $\cH_{i,c}(\cdot)$ and $\cH_{i,w}(\cdot)$
and by substituting $x'_{1:T}=(x'(t),\cdots,x'(T))$ back to \eqref{eqn:objective}, 
we see that the objective value in \eqref{eqn:objective}
with $x'(t)$ as the solution
 is even smaller, which is in contradiction to the assumption 
 that $x^*_{1:T}$ is optimal. Therefore, for the transformed problem, the optimal objective value has to be the same as the original one, and the action $x_{1:T}^*=x'_{1:T}$ is the optimal solution.
\end{proof}

Based on
the equivalence
of the new transformed problem to the original problem,
we now focus on solving the transformed problem \eqref{eqn:objective_dual}--\eqref{eqn:constraint_gateway_dual_2}.
The two added constraints \eqref{eqn:constraint_gateway_dual_1}
and \eqref{eqn:constraint_gateway_dual_2} still involve
all the decisions over $T$ time slots.
To remove the temporal coupling, 
we consider the Lagrangian form of the transformed
problem \eqref{eqn:objective_dual}--\eqref{eqn:constraint_gateway_dual_2}.
For the convenience of notation,
we first define
\begin{align}
\cH_c(z_c(t)) & = [\cH_{1,c}(z_c(t)),\cdots, \cH_{N,c}(z_c(t))],\\
\cH_w(z_w(t))& = [\cH_{1,w}(z_w(t)),\cdots, \cH_{N,w}(z_w(t))],\\
\cC(t)& = [ c_{1,t}(x(t)), \cdots,  c_{N,t}(x(t))  ],\\
\cW(t)& = [ w_{1,t}(x(t)), \cdots,  w_{N,t}(x(t))  ].
\end{align}
Then, subject to the constraints \eqref{eqn:constraint_delay_assignment}\eqref{eqn:constraint_datacenter_capacity}\eqref{eqn:constraint_gateway},
we write the Lagrangian as follows:
 \begin{equation}\label{eqn:lagrange_functoin}
\begin{aligned}
    &\mathcal{L}(x_{1:T}, z_{c,1:T}, z_{w, 1:T}, \kappa) \\
    =\;& \frac{1}{T}\left( \sum_{t=1}^T g_t(x(t)) +  \mu_c\| \cH_{c}({z}_{c}(t)) \|_{\infty} + \mu_w\| \cH_{w}({z}_{w}(t)) \|_{\infty} \right) \\
    & +  \langle \kappa, \begin{bmatrix} \frac{1}{T} \cdot \left(\sum_{t=1}^T\cC_{t}(x(t)) - \sum_{t=1}^T{z}_{c}(t) \right) \\ \frac{1}{T} \cdot \left(\sum_{t=1}^Tw_{t}(x(t)) - \sum_{t=1}^T{z}_{w}(t)\right)  \end{bmatrix} \rangle
\end{aligned}
\end{equation}
where $\kappa$ is the Lagrangian multipliers associated
with the constraints \eqref{eqn:constraint_gateway_dual_1}
and \eqref{eqn:constraint_gateway_dual_2},
and $\langle a,b\rangle$ denotes the inner product of
two vectors $a$ and $b$.

By solving the problem \eqref{eqn:lagrange_functoin} 
online subject to the constraints \eqref{eqn:constraint_delay_assignment}\eqref{eqn:constraint_datacenter_capacity}\eqref{eqn:constraint_gateway}, we would obtain the optimal GLB decisions  
if the optimal Lagrangian multiplier  $\kappa$ were provided.
Nonetheless, $\kappa$ can only be estimated with online information.
Based on this insight,
we sequentially update $\kappa$ using dual mirror descent (DMD)  \cite{OMD_book}
based on online information
and obtain GLB decisions $x(t)$ for $t=1,\cdots, T$.

We describe the algorithm in Algorithm~\ref{alg:dmd_with_equity}.
More specifically, at time $t$, we receive the cost functions and optimize the action $x(t)$ and auxiliary variable $z(t)=(z_c(t),z_w(t))$ according to the current estimate of dual variable $\kappa_t$. These variables are optimized in Line~4 and Line~5, respectively. 
The insight is that the estimated dual variable $\kappa_t$ controls the adjusted penalty for the GLB action $x(t)$ based on how much
the cumulative actual carbon and water footprints
have deviated from the targets $z(t)=(z_c(t),z_w(t))$.

We update the dual variable $\kappa_t$ using DMD.
More concretely, by taking the subgradient of $\kappa$ with respect to the Lagrange function and using the online information at time $t$, we obtain a stochastic gradient estimate of $\kappa_t$. 
In Line~7, the vector $d_t$ is set as the opposite direction to the gradient of $\kappa_t$ in order to minimize the Lagrange function. 
Finally, the updated dual variable $\kappa_{t+1}$ is obtained with Bregman projection using a reference function $h(\cdot)$ which is
differentiable
and strongly convex. For example, a common choice of
the reference function is $h(a)=\frac{1}{2}\|a\|^2$, which results
in additive updates of the dual variable estimate $\kappa_t$ \cite{OnlineAllocation_DualMirroDescent_Google_OperationalResearch_2022_doi:10.1287/opre.2021.2242}.

\begin{algorithm}[!t]  
	{\bf Input:} Initial Lagrange multiplier $\kappa_1 \in \RR^{2N}_{\geq 0}$, reference  function $h(\cdot): \mathbb{R}^{2N}\rightarrow  \mathbb{R}$, total length of horizon $T$ and learning rate $\eta$ \\
	\For{$t=1,\ldots,T$}{
		Receive the cost function of energy, carbon and water as $g_t(\cdot)$, $c_t(\cdot)$ and $w_t(\cdot)$, the action constraint $\mathcal{X}_t=\{x|x \text{ satisfies }\eqref{eqn:constraint_delay_assignment}\eqref{eqn:constraint_datacenter_capacity}\eqref{eqn:constraint_gateway}\}$.\\
		Make the primal decision
		\begin{equation*}
		{x}(t) = \arg\min_{x(t)\in \cX_t}\{ g_t( x(t) )  +  \kappa_t^\top \cdot \begin{bmatrix}   \cC_t(x(t))  \\   \cW_t(x(t))  \end{bmatrix}  \} \ ,
		\end{equation*}\\
		
		Determine the auxiliary variable:
		\begin{equation*}
              \begin{aligned}
                  \{z_c(t),z_w(t)\} = &\arg \min_{z_c \in \cZ_c, z_w \in \cZ_w}\{ \mu_c \| \cH_c(z_c)\|_{\infty} \\
                    &+  \mu_w \| \cH_w(z_w) \|_{\infty}  - \kappa_t^\top  \begin{bmatrix} z_c\\ z_w \end{bmatrix} \}\ ,
              \end{aligned}
		\end{equation*}\\
		
		Obtain a stochastic subgradient of $\kappa_t$: 
                \begin{equation*}
                    d_t = \begin{bmatrix} z_c(t)\\ z_w(t)\end{bmatrix} - \begin{bmatrix}   \cC_t(x(t))  \\   \cW_t(x(t))  \end{bmatrix} \ .
                \end{equation*}
		
		 Update the dual variable by mirror descent:
		\begin{equation*}\label{eqn:dmd_dual_update}
           \kappa_{t+1} = \arg\min_{\kappa \in \mathcal{R}^{2N}_{\geq 0}} \langle d_t, \kappa \rangle + \frac{1}{\eta} V_h(\kappa, \kappa_t) \ ,
            \end{equation*}
    \color{black} where $V_h(x,y)=h(x)-h(y)-\nabla h(y)^\top (x-y)$ is the Bregman divergence.
}	\caption{Online GLB for Environmentally Equitable AI (\ouralgonline)}\label{alg:dmd_with_equity}
\end{algorithm}

Next, we analyze  \ouralg in terms
of the cost objective in \eqref{eqn:objective}.
\begin{theorem}\label{thm:cost_bound}
By initializing $\kappa_1 \in \RR^{2N}_{\geq 0}$ as a zero vector,
considering the reference function $h(a) = \frac{1}{2}\|a\|^2$,
    and denoting the GLB actions as $x_{1:T}=(x(1),\cdots, x(T))$
    and the overall cost  defined in \eqref{eqn:objective}
    as $\text{cost}(x_{1:T})$,
    we have the following:
    \begin{equation}\label{eqn:total_cost_bound}
        \text{cost}(x_{1:T}) \leq \text{cost}(x_{1:T}^*) + \eta BT + C \sqrt{\frac{2}{T}(B + \frac{M}{\eta}D)}
    \end{equation}
    where $\text{cost}(x_{1:T}^*)$
    is the minimum cost given by optimal offline algorithm,
    $\eta>0$ is the learning rate, $c_m$ and $w_m$ are the maximum possible gradients of carbon and water footprints in \eqref{eqn:carbon_footprint} and \eqref{eqn:water_footprint}, $M=\max_{i \in \cN} M_i$ is the maximum processing capacity of all data centers, $\theta_m$ is the maximum gradient of $H_{i,c}(\cdot)$ and $H_{i,w}(\cdot)$, 
    $B=\frac{N}{2}
    \left[\max_{i\in \cN} \bar{z}^2_{i,c}\right]+
    \frac{N}{2}
    \left[\max_{i\in \cN} \bar{z}^2_{i,w}\right]$ (in which
    $\bar{z}_{i,c}$ and $\bar{z}_{i,w}$ are the maximum
    possible per-time carbon and water footprints
    in data center $i$, respectively),
    $C=\theta_m(\mu_c + \mu_w)$ and $D = \theta_m(\mu_c c_m + \mu_w w_m)$, respectively.
    Moreover, by setting the learning rate $\eta=\mathcal{O}(1/T)$, we have
    \begin{equation}\label{eqn:total_cost_bound_2}
        \text{cost}(x_{1:T}) \leq \text{cost}(x_{1:T}^*) + \mathcal{O}(1).
    \end{equation}
\end{theorem}

Theorem~\ref{thm:cost_bound} bounds the gap between  \ouralg
and the optimal offline algorithm in terms of the overall cost defined in
\eqref{eqn:objective}. The constants
$B$ and $D$ are problem-specific  and
naturally increase as the input range is larger.
In addition,
the gap depends on the choice of the learning rate $\eta$.
Specifically, by increasing $\eta$, \ouralg updates the dual variable
$\kappa$ by more aggressively following the stochastic gradient
(Line~6 in Algorithm~\ref{alg:dmd_with_equity}).
This can introduce greater drifts due to the ``forgetting'' of
the past time slots and hence increases the term $\eta B T$.
On the other hand, a larger $\eta$ can reduce the time steps
needed for updating the dual variable
to track the optimal dual variable, and hence reduce
the term $D\sqrt{\frac{2}{T}(B + \frac{M}{\eta}D)}$.
Thus,  by setting the learning rate $\eta=\mathcal{O}(1/T)$
to balance the two terms, we can have an
$\mathcal{O}(1)$ cost gap.
Note that, without further stochastic assumptions (e.g., all the inputs follow an independent and identical distribution), 
 eliminating the $\mathcal{O}(1)$ cost gap between \ouralg and the optimal offline algorithm remains an open challenge in the literature  \cite{OnlineAllocation_DualMirroDescent_Google_OperationalResearch_2022_doi:10.1287/opre.2021.2242,Neely_Universal}.
For example, in a relevant context of online budget allocation,
having a zero cost gap is impossible in general adversarial settings
that we consider  \cite{OnlineAllocation_DualMirroDescent_Google_OperationalResearch_2022_doi:10.1287/opre.2021.2242}. 
Importantly, as is shown in our experimental results
(Section~\ref{sec:results}),
\ouralg demonstrates a strong empirical performance
even compared to the optimal offline
algorithm.

\section{Experiments}\label{sec:experiment}

In this section, we report on experiments of different GLB algorithms using trace-based simulations.
Our results demonstrate that 
\ouralg has a great potential to effectively address AI's environmental inequity that
would otherwise be potentially amplified by other GLB algorithms.
Importantly, the empirical cost performance of \ouralg is
close to the optimal offline equity-aware GLB, complementing
our theoretical analysis of \ouralg in Theorem~\ref{thm:cost_bound}.

\subsection{Methodology}\label{sec:methodology}
As detailed information about AI system and workload
settings is typically proprietary,
 we run simulations
by scaling up workload traces collected from public sources and considering synthetic data center settings that approximate realistic scenarios.  This is in line with the prior GLB literature \cite{Liu:2011:GGL:1993744.1993767,RaoLiuXieLiu_2010,Gao:2012:EG:2377677.2377719,Shaolei_Water_SpatioTemporal_GLB_TCC_2018_7420641}.
Next, we describe the default setup of our experiments,
which will later be varied for sensitivity studies.

\subsubsection{Workload Trace}
To obtain the workload trace, we extract the GPU power usage 
data from \cite{ML_Carbon_Bloom_176B_Sasha_Luccioni_arXiv_2022_luccioni2022estimating} 
for the server cluster hosting the large language model BLOOM
over an 18-day period (between September 23 and
 October 11 in 2022). Because there is only a single workload trace provided for BLOOM
 in \cite{ML_Carbon_Bloom_176B_Sasha_Luccioni_arXiv_2022_luccioni2022estimating},
 we follow the data augmentation method in \cite{Carbon_GenerativeAI_AndrewChien_Chicago_HotCarbon_2023_10.1145/3604930.3605705} and distribute the workload trace to the 10 gateways (plus a small perturbation to account for different time zones).
As in \cite{ML_Carbon_Bloom_176B_Sasha_Luccioni_arXiv_2022_luccioni2022estimating},
 we directly quantify the amount of workload using power demand.
We also scale up the workload trace to let the maximum workload
match our data center power capacity as introduced below. 
The 18-day workload trace will be also be extended using data augmentation techniques to evaluate different GLB algorithms
over a longer timescale (Section~\ref{sec:experiment_longer_timescale}).

\subsubsection{Data Centers}\label{sec:simulation_setup_datacenter}

We consider a set of 10 geo-distributed data centers,
including four in the U.S. (Virginia, Georgia, Texas, and Nevada),
four in Europe (Belgium, the Netherlands, Germany, and Denmark),
and two in Asia (Singapore and Japan). These locations
are all a large presence of data centers, including Google's
data centers \cite{ML_CarbonFoorptint_Google_JeffDean_Journal_2022_9810097}.
The details of data center locations are available in the appendix.

 Assuming that there are 10 gateways corresponding
to the 10 data center locations, we consider two scenarios:
(1) \textbf{full GLB flexibility}: the workloads can be flexibly dispatched
from any gateway to any data center;
and (2) \textbf{partial GLB flexibility}: the workloads arriving at a
gateway can only be dispatched to a certain subset of data centers.
As shown in recent studies \cite{Carbon_GenerativeAI_AndrewChien_Chicago_HotCarbon_2023_10.1145/3604930.3605705}, even cross-continent AI workload placement
only marginally increases the end-to-end latency without degrading
service quality.
Thus, the ``full GLB flexibility'' scenario
is already feasible in practice,
whereas the ``partial GLB flexibility'' scenario 
accounts for various other constraints such
as strict latency and bandwidth.

For processing AI inference workloads, we assume that
each data center houses a cluster of 500 
homogeneous servers. Each server is equipped with four
NVIDIA A100 GPUs and has a maximum total power of 2 kW.
Thus, excluding the network switches
and servers for other services beyond the scope
of our study, each data center has a maximum server power of 1 MW 
for AI inference. 

We set the data center PUE as 1.1, which is consistent
with the state-of-the-art PUE value
with efficient operation
\cite{Google_SustainabilityReport_2022,ML_CarbonFoorptint_Google_JeffDean_Journal_2022_9810097}. 
For simplicity, we use the actual carbon footprint and water footprint
to measure the regional environmental impact
(i.e., $\mathcal{H}_{i,c}(x)=x$
and $\mathcal{H}_{i,w}(x)=x$ in \eqref{eqn:objective}).

\subsubsection{Energy Price, Carbon Intensity, and WUE}

We collect hourly energy prices for the 10 data centers over
the same 18-day period as our workload trace. Specifically,
for each data center in Europe and Asia, we collect the hourly 
country-level energy prices from  \cite{Energy_Price_Fuel_Data_International_IEA_website}.
For the U.S. data centers, we collect the hourly
energy prices from their respective ISOs \cite{Water_EnergyData_EIA_Website}.

For each of
the U.S. data centers, we
collect the  
state-level hourly energy fuel mix data \cite{Water_EnergyData_EIA_Website}
and calculate the indirect WUE
based on the fuel mix by following \cite{Gao:2012:EG:2377677.2377719}
and \cite{Shaolei_Water_SpatioTemporal_GLB_TCC_2018_7420641}, respectively.
The carbon intensity and energy water intensity factor (EWIF)
for each fuel mix
are chosen based on \cite{Gao:2012:EG:2377677.2377719} and \cite{Shaolei_Water_AI_Thirsty_arXiv_2023_li2023making}.
We do not have free access to the hourly energy fuel mix
data for our data center locations in Europe and Asia \cite{Energy_Price_Fuel_Data_International_IEA_website}.
Thus, we generate synthetic hourly fuel mixes for these
locations based on the U.S. data. 
Besides, the hourly carbon intensity of each datacenter is obtained from \cite{electric_map}, where the US locations are ISO level and the Europe and Asia locations are country-level carbon intensity. 
The details
are available in the appendix. 

To model the on-site WUE, we assume that the data centers use
cooling towers for heat rejection, which are common in the industry
(even in water-stressed regions like Arizona \cite{Water_DataCenterEnergy_Tradeoff_Arizona_Real_Measurement_WUE_Monthly_2022_KARIMI2022106194}). We collect
the hourly weather data from \cite{Water_WeatherData_Website} 
for the airports closest to each of our data center locations,
and then obtain the wet bulb temperature from the dry bulb temperature
and relative humidity based on \cite{Water_WetBulbTemperature_Meyer2019}.
Next, we calculate the on-site WUE using the empirical
formula in terms of the wet-bulb temperature presented in \cite{Shaolei_Water_SpatioTemporal_GLB_TCC_2018_7420641}.
While assuming cooling towers for rejecting heat into the
outside environment,
our study can be easily adapted to air-side economizers,
which use water for humidity control or
when the outside dry bulb temperature is high \cite{Facebook_Water_2023_meta}.

\begin{table*}[!t]
    %\scriptsize
    \footnotesize
    \caption{Comparison of different GLB algorithms. The metric ratio is the maximum water or carbon footprint divided by the average.
    The results of \ouralg with
    the learning rate $\eta=1.7 \times 10^{-4}$ are bolded. }   \label{table:comparison_all_results}
    \centering
    \begin{tabular}{c|c|c|c|c|c|c|c|c|c|c|c} 
    \toprule
    \textbf{GLB}& \multicolumn{2}{c|}{\multirow{2}{*}{\textbf{Metric}}} & \multicolumn{7}{c}{\textbf{Algorithm}} \\ 
    \cline{4-12}
    \textbf{Flexibility}        & \multicolumn{2}{c|}{}                        &  \costalg       &  \carbonalg     & \wateralg      & \costcarbonalg    & \allalg      & \distancealg         & {\ouralgoffline} & \ouralgmpc & \textbf{\ouralg}\\ 
    \hline
    \multirow{7}{*}{Full}            & \textbf{Energy} (US\$)                     & avg        & 29170 & 47708 & 56184 & 33735 & 32466 & 47038 & 36106 & 36199 &\textbf{37643}\\ 
    \cline{2-12}
    & \multirow{3}{*}{\textbf{Water} ($\text{m}^3$)} & avg      &  1525.1 & 1396.2 & 1243.9 & 1486.3 & 1426.7 & 1446.7 & 1431.8 & 1444.4 &\textbf{1448.7}\\ 
    \cline{3-12}
    &  & max   & 2607.5 & 2671.6 & 2010.4 & 2675.7 & 2358.0 & 2090.9 & 1705.2 & 1898.3 & \textbf{1928.0} \\
    \cline{3-12}
    & & \textbf{max/avg} & 1.71 & 1.91 & 1.62 & 1.80 & 1.65 & 1.45 & 1.19 & 1.31  & \textbf{1.33}\\
    \cline{2-12}
    & \multirow{3}{*}{\textbf{Carbon} (ton)}   & avg  & 118.17 & 90.50 & 103.17 & 100.75 & 105.83 & 111.80 & 102.59 &  105.06 &\textbf{105.70}\\ 
    \cline{3-12}
    &  & max        & 205.10 & 166.79 & 224.49 & 166.82 & 171.50 & 143.46 & 128.79 & 134.68 &\textbf{139.39} \\  
    \cline{3-12}
    & & \textbf{max/avg} & 1.74 & 1.84 & 2.18 & 1.66 & 1.62 & 1.28 & 1.26 & 1.28 & \textbf{1.32} \\
    \hline
    \multirow{7}{*}{Partial}  & \textbf{Energy} (US\$)   & avg &  29659 & 47694 & 53976 & 33729 & 32822 & 47038 & 36013 & 37210 &\textbf{37768} \\  
    \cline{2-12}
    & \multirow{3}{*}{\textbf{Water} ($\text{m}^3$)} & avg   & 1524.1 & 1415.5 & 1249.9 & 1490.4 & 1420.1 & 1446.7 & 1440.6 & 1446.6 & \textbf{1450.1} \\  
    \cline{3-12}
    &   & max  & 2616.1 & 2700.3 & 2028.4 & 2668.8 & 2344.7 & 2090.9 & 1777.3 & 1891.4 &\textbf{1929.0} \\ 
    \cline{3-12}
    & & \textbf{max/avg} & 1.72 & 1.91 & 1.62 & 1.79 & 1.65 & 1.45 & 1.23 & 1.31 & \textbf{1.33}\\
    \cline{2-12}
    & \multirow{3}{*}{\textbf{Carbon} (ton)}   & avg  & 117.52 & 92.22 & 106.07 & 102.13 & 106.53 & 111.80 & 103.31 & 105.11 &\textbf{105.65}\\   
    \cline{3-12}
    &   & max  &  205.77 & 168.42 & 224.49 & 166.44 & 177.43 & 143.46 & 133.45 & 135.42 & \textbf{139.89} \\ 
    \cline{3-12}
     & & \textbf{max/avg} & 1.75 & 1.83 & 2.12 & 1.63 & 1.67 & 1.28 & 1.29 & 1.29 & \textbf{1.32}\\
    \bottomrule
    \end{tabular}
\end{table*}

\subsubsection{Offline and Online Optimization}

Assuming complete knowledge of future information,
we first use offline optimization in order to quantify the maximal potential of equity-aware GLB to address AI's environmental inequity. We then use our online algorithm \ouralg to demonstrate how much of that potential can be realized.  

In the offline case, we consider hourly GLB decisions and use \emph{cvxpy} to solve \eqref{eqn:objective}--\eqref{eqn:constraint_gateway} offline based
on the complete information about all the future workload arrivals,
energy prices, carbon intensity, and WUE values. 
We refer to this offline algorithm as \textbf{\ouralgoffline}.
It takes about 3 minutes on a desktop with Intel i7-9700K CPU and 16GB RAM
to solve the problem for an 18-day simulation in our experiments.
The weight hyperparameters in \eqref{eqn:objective} are set as $\mu_c=1500$ \$/ton and $\mu_w=60$ \$/$\text{m}^3$. Note that these hyperparameters 
are only used to adjust the relative
importance of different cost terms
in the optimization process and do not reflect
the true monetary costs of carbon or water footprints.

In the online case, we use \ouralg to optimize the GLB decisions according to the sequentially 
revealed workload arrivals, energy price, carbon intensity, and water efficiency information. 
It takes around 30 seconds on the same machine to calculate GLB decisions for the 18-day simulation.
{
Besides, we compare \ouralg with another online policy, \ouralgmpc, which leverages model predictive control (MPC). Specifically, \ouralgmpc optimizes the objective in Eqn~\eqref{eqn:objective} over a receding 24-hour horizon, utilizing predictions of future workloads, energy prices, carbon intensities, and water efficiencies.
}
For a fair comparison, the hyperparameters $\mu_c$ and $\mu_w$  by \ouralg and \ouralgmpc are chosen as the same values as the offline optimizer \ouralgoffline. 

\subsubsection{Metrics}\label{sec:experiment_metrics}
We evaluate our equity-aware GLB algorithms using the following metrics:
\textit{average energy cost}, the total energy
cost throughout the 18-day period divided
by 10 data center locations;
\textit{average carbon/water footprint}, the total carbon/water footprint throughout
the 18-day period divided by 10 data center locations;
and the \textit{maximum regional carbon/water footprint}
over the 18-day period among the 10 data center locations.
If scaled up
by a factor of 10, the average value is equivalent
to the total value.  We also include
the maximum regional carbon/water footprint to the average value
to reflect the level of environmental equity, i.e.,
the smaller $max/avg$, the more equitable, and 
 $max/avg=1$ means all the regions have the same environmental cost in terms of the carbon/water footprint.

\subsubsection{Baseline Algorithms}\label{sec:baseline}

We consider the following GLB-related algorithms 
for comparison. 

\begin{itemize}

\item \costalg: This algorithm is based on \cite{RaoLiuXieLiu_2010,Liu:2011:GGL:1993744.1993767,DataCenter_CuttingElectricBill_MIT_Sigcomm_2009_10.1145/1592568.1592584} and only minimizes the total energy cost. It is a special
case of \ouralg by setting $\mu_c=0$ and $\mu_w=0$ in \eqref{eqn:objective}.

\item \carbonalg: Minimization of the total
carbon footprint.

\item \wateralg: Minimization of  the total
water footprint.

\item \costcarbonalg: This algorithm is based on \cite{Gao:2012:EG:2377677.2377719} and minimizes the weighted
sum of the total energy cost and  carbon footprint.

\item \allalg: This algorithm is based on \cite{Shaolei_Water_SpatioTemporal_GLB_TCC_2018_7420641}
and minimizes the weighted
sum of the total energy cost, carbon footprint, and water footprint.

\item\distancealg: This algorithm is a special
case of GLB and directly routes workloads
from each gateway to its nearest data center. It is commonly
used in practice as a default baseline algorithm \cite{Gao:2012:EG:2377677.2377719,DataCenter_CuttingElectricBill_MIT_Sigcomm_2009_10.1145/1592568.1592584}.
\end{itemize}

Without considering equity-related costs, the GLB decisions in these algorithms are not coupled
over time and hence can be optimally obtained online.
 The weights for carbon and water (if applicable)
in \costcarbonalg and \allalg are set such that their respective total carbon
and water footprints are smaller than those of \ouralgoffline.

\begin{figure*}
\centering
\subfloat{
\includegraphics[width=0.8\linewidth]{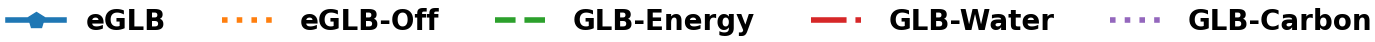}
\label{fig:subfig1}
} \par%
\setcounter{subfigure}{0}
\subfloat[S][Total weighted cost]{
\includegraphics[width=0.32\linewidth]{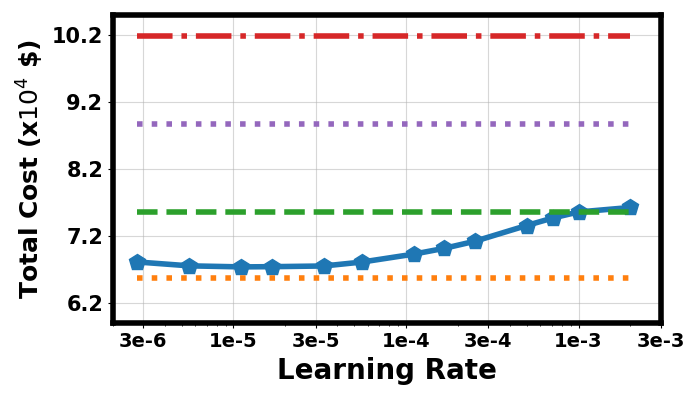}
\label{fig:full_total_cost}
}
\subfloat[S][Average energy cost]{
\includegraphics[width=0.32\linewidth]{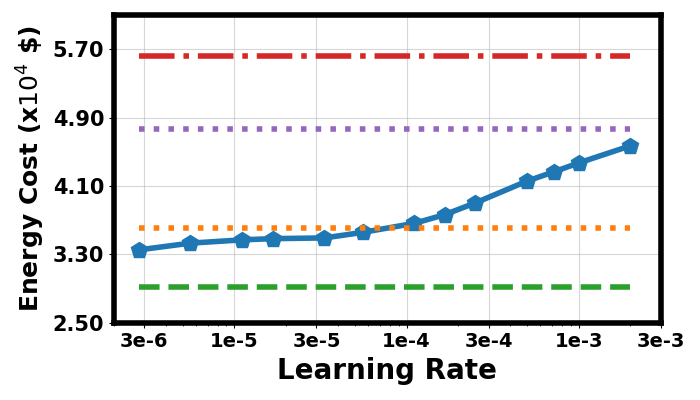}
\label{fig:full_avg_price}
}%
\subfloat[S][Maximum water]{%
\includegraphics[width=0.32\linewidth]{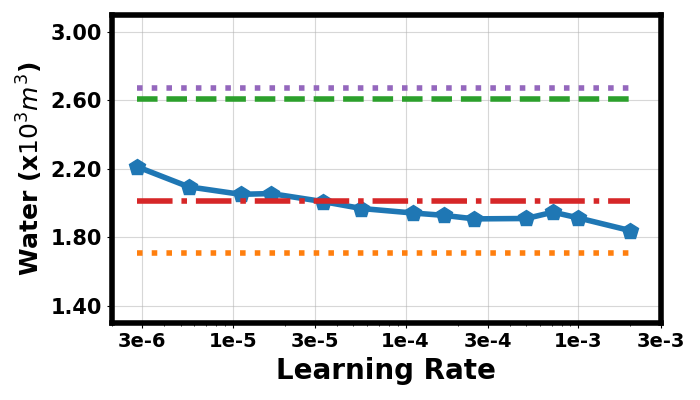}
\label{fig:full_max_water}
}%
\par%
\subfloat[][Average water]{%
\includegraphics[width=0.32\linewidth]{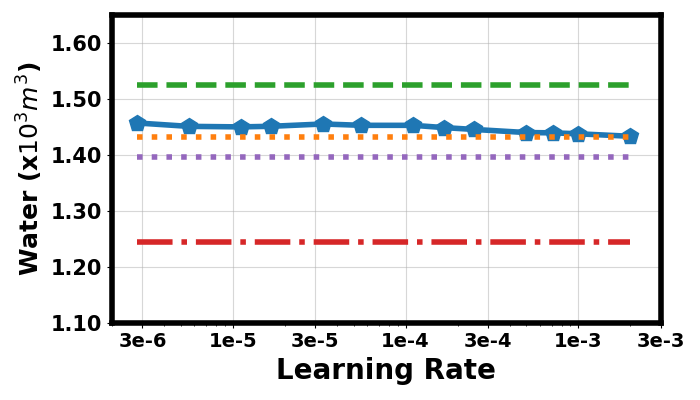}
\label{fig:full_avg_water}
}%
\subfloat[][Maximum carbon]{%
\includegraphics[width=0.32\linewidth]{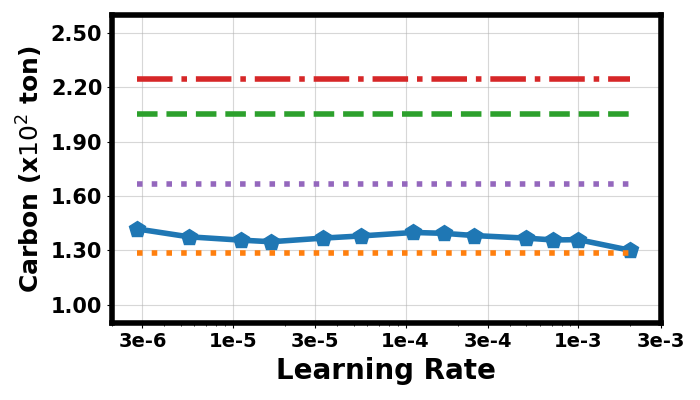}
\label{fig:full_max_carbon}
}%
\subfloat[][Average carbon]{%
\includegraphics[width=0.32\linewidth]{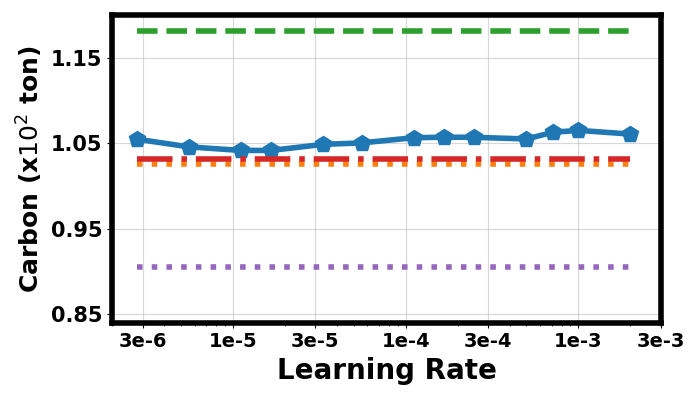}
\label{fig:full_avg_carbon}
}%
\par%
\caption{The energy cost, carbon and water footprint of \ouralg with different learning rates $\eta$ under  full GLB flexibility. The
results for \ouralgoffline, \costalg, \carbonalg,
and \wateralg are shown for comparison.}\label{fig:comparison_full}
\end{figure*}%

\subsection{Results}\label{sec:results}
We show our results in Table~\ref{table:comparison_all_results} by considering two different scenarios: GLB with full or partial flexibility.
Our results highlight that 
\ouralg can improve AI's environmental equity 
by reducing the environmental impact on the most disadvantaged region while still keeping the average environmental footprint and energy cost close to or even lower than those of alternative GLB algorithms.
In addition, 
the empirical performance of \ouralg is
close to the optimal offline algorithm \ouralgoffline.
Next, we discuss our results in detail.

\subsubsection{GLB with Full Flexibility}

We first consider the full-flexibility scenario
in which the workloads can be dispatched to any data center.
Among all the algorithms, \ouralgoffline has the lowest carbon and water
footprints for the most disadvantaged regions with complete information of workload, carbon and water footprints. 
Meanwhile, the average energy cost, carbon footprint, and water footprint
of \ouralgoffline are comparable to or even lower than the other GLB algorithms. 
Thus, \ouralgoffline has almost
the lowest ``\emph{maximum to average}'' ratio in terms
of both the carbon footprint and water footprint,
effectively reducing the regional disparity and improving environmental equity.

Interestingly, while \costalg, \carbonalg, and \wateralg can minimize
the total energy cost, carbon footprint, and water footprint, respectively,
they amplify the environmental inequity compared to \distancealg.
This is due to the inequity \emph{unawareness} of these algorithms ---
their aggressive exploitation of certain regions may come at the
cost of harming these regions in terms of environmental impacts. For example, \costalg exploits the cheaper
energy price of Texas by assigning more workloads
to this region, but this can result in
a disproportionately high environmental
footprint in Texas due to its
 worse carbon intensity and/or WUE than some other regions. 
While \costcarbonalg and \allalg can balance the energy cost and environmental footprints in terms of the average/total metric, 
they can still result
in disproportionately high environmental burdens on the already-disadvantaged regions
due to the unawareness of equity. This is similar
to algorithmic unfairness against disadvantaged individuals or user groups caused by an AI model that purely minimizes
the average loss \cite{Fair_MinimaxPareto_Fairness_ICML_2020_10.5555/3524938.3525565,Fair_MiniMaxGroupFairness_AaronRoth_Amazon_2021_Diana2021_10.1145/3461702.3462523}.

While the prior studies \cite{Gao:2012:EG:2377677.2377719,Shaolei_Water_SpatioTemporal_GLB_TCC_2018_7420641} have demonstrated that the  total carbon
footprint and water footprint are often in tension with the energy cost,
our results further add that environmental equity may not
be cost-free either. Nonetheless, by balancing the energy
cost and environmental equity as formulated in \eqref{eqn:objective},
the price we pay for equity can be reasonably low.

\textbf{\ouralg vs. \ouralgoffline}. In comparison to the offline optimizer, \ouralg only has access to online causal information. This naturally leads to worse performance than if full information were available; however, the carbon and water footprints for the most disadvantaged regions in \ouralg are still better than all other GLB algorithms. 
In Table~\ref{table:comparison_all_results}, the maximum/average carbon and water footprints of \ouralg are mostly within $10\%$ of the offline optimal solutions from \ouralgoffline, which demonstrates the empirical effectiveness of \ouralg and complements our theoretical analysis in Theorem~\ref{thm:cost_bound}.

\textbf{\ouralg vs. \ouralgmpc}. Unlike the offline optimal policy \ouralgoffline, \ouralgmpc is limited to future information within the 
next 24-hour prediction window. This limitation reflects real-world constraints, where perfect foresight of the entire future is unavailable.  As expected, \ouralgmpc outperforms \ouralg (with sequentially revealed information) in
terms of the average energy cost, water footprint, and carbon footprint.
Interestingly, the maximum-to-average ratios for carbon and water footprints are similar between \ouralgmpc and \ouralg, despite \ouralgmpc's 24-hour prediction window. This observation emphasizes the fundamental challenge of achieving long-term environmental equity when faced with limited predictions of the future. 

\begin{figure*}[!t]
\centering
\setcounter{subfigure}{0}
\subfloat[S][Average energy cost]{
\includegraphics[width=0.19\linewidth]{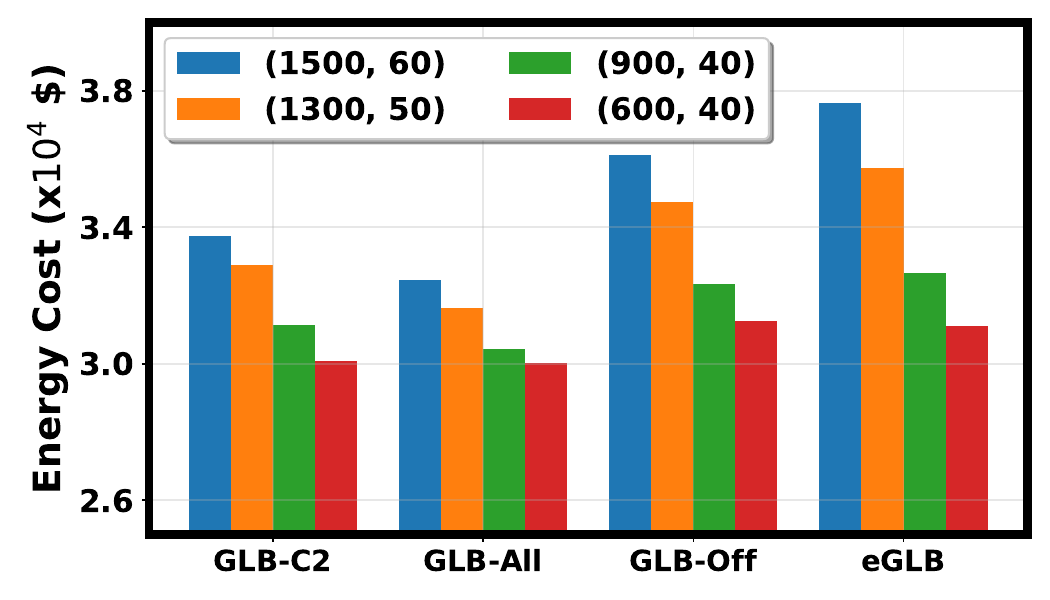}

}%
\subfloat[S][Maximum water]{%
\includegraphics[width=0.19\linewidth]{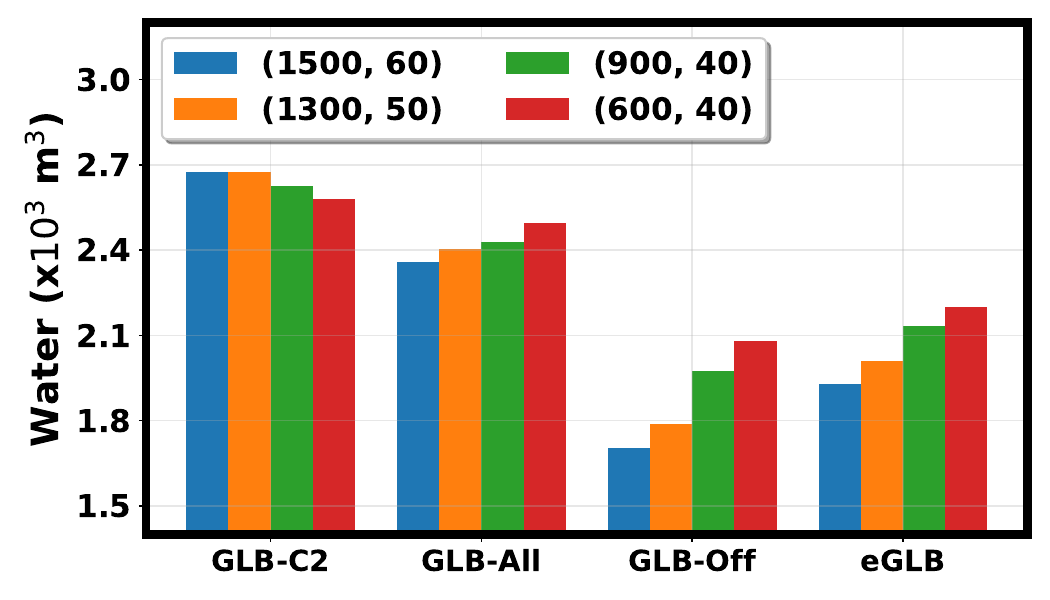}

}%
\subfloat[][Maximum carbon]{%
\includegraphics[width=0.19\linewidth]{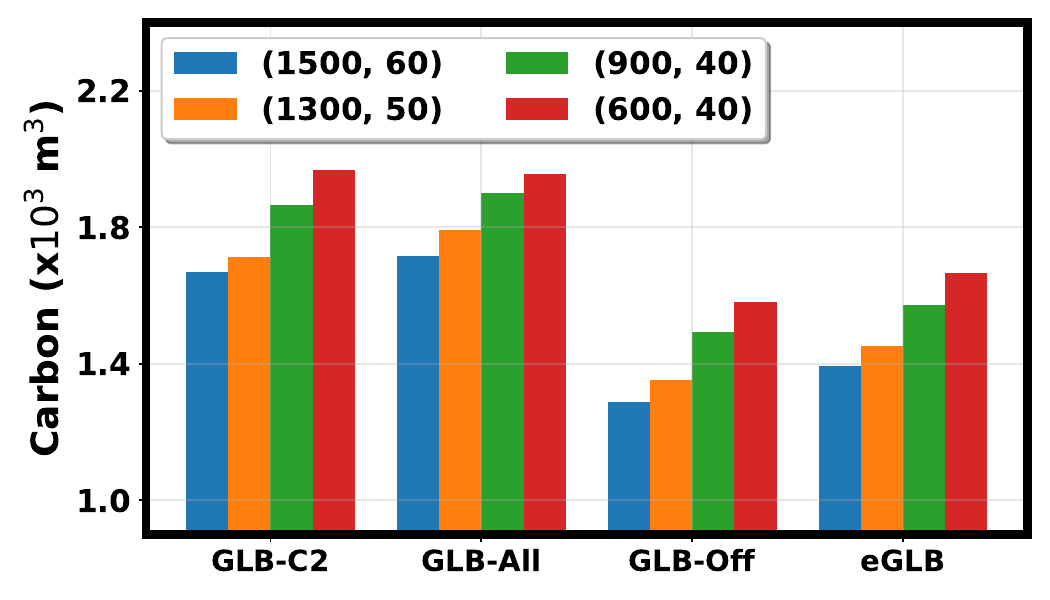}
}%
\subfloat[S][Max/avg  for water]{%
\includegraphics[width=0.19\linewidth]{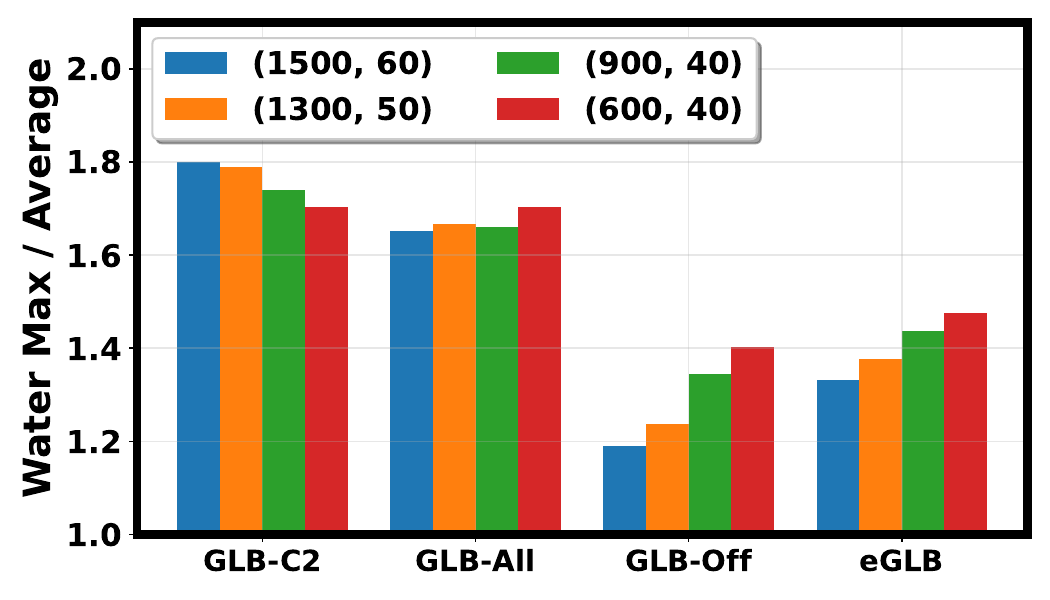}

}%
\subfloat[][Max/avg for  carbon]{%
\includegraphics[width=0.19\linewidth]{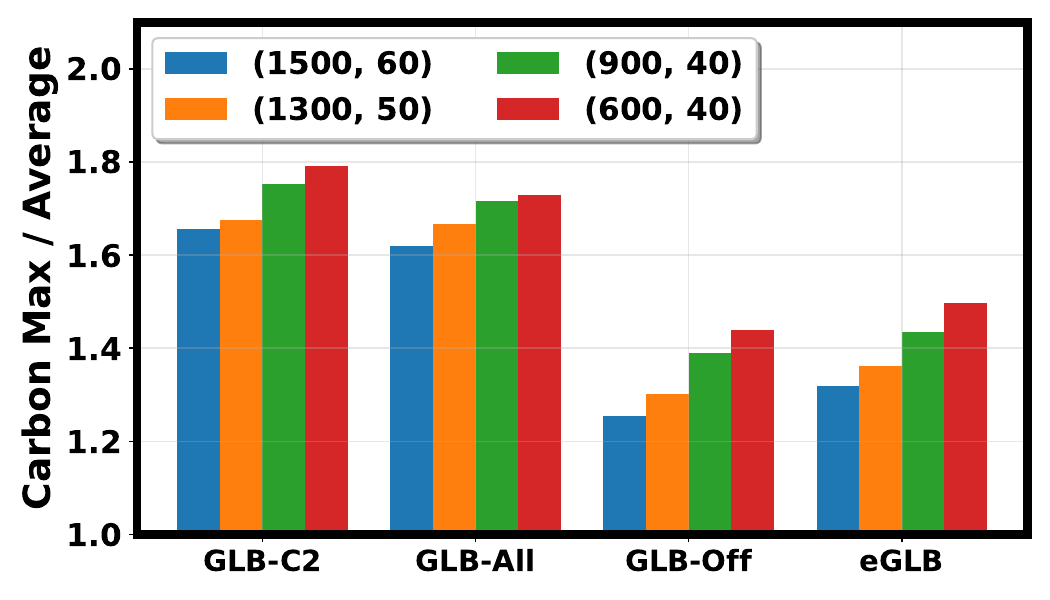}
}
\par%
\caption{The energy cost, carbon footprint, and water footprint of \ouralg with different $(\mu_w,\mu_c)$ shown in the legend under full GLB flexibility. The results for \ouralgoffline, \costalg, \carbonalg,
and \wateralg are also shown for comparison.}\label{fig:compare_weight}
\end{figure*}

\begin{table*}[!t]
    
    \footnotesize
    \caption{Comparison of different GLB algorithms. The default workload trace is augmented to 180 days to evaluate the long-term impact of different GLB algorithms. The results of \ouralg with
    the learning rate $\eta=1.7 \times 10^{-4}$ are bolded.}   \label{table:long_sequence}
    \centering
    \begin{tabular}{c |c|c|c|c|c|c|c|c|c|c} 
    \toprule
    \textbf{GLB} & \multicolumn{2}{c|}{\multirow{2}{*}{\textbf{Metric}}} & \multicolumn{7}{c}{\textbf{Algorithm}} \\ 
    \cline{4-11}
    \textbf{Flexibility}     & \multicolumn{2}{c|}{}                        &  \costalg       &  \carbonalg     & \wateralg      & \costcarbonalg    & \allalg      & \distancealg         & {\ouralgoffline} & \textbf{\ouralg}\\ 
    \hline
     \multirow{7}{*}{Full}  & \textbf{Energy} (US\$)                     & avg        & 279620 & 454608 & 539847 & 326104 & 312372 & 450992 & 341998 & \textbf{359433}\\ 
    \cline{2-11}
    & \multirow{3}{*}{\textbf{Water} ($\text{m}^3$)} & avg      &  14329.6 & 12992.8 & 11694.2 & 13822.4 & 13338.4 & 13584.9 & 13439.3 & \textbf{13591.5}\\ 
    \cline{3-11}
    &  & max   & 23753.4 & 24779.5 & 19478.0 & 25154.2 & 21307.6 & 19662.3 & 16339.6 & \textbf{18199.0} \\
    \cline{3-11}
    & & \textbf{max/avg} & 1.66 & 1.91 & 1.67 & 1.82 & 1.60 & 1.45 & 1.22 & \textbf{1.34}\\
    \cline{2-11}
    & \multirow{3}{*}{\textbf{Carbon} (ton)}   & avg  & 1098.29 & 830.66 & 947.89 & 925.28 & 975.76 & 1035.97 & 951.91 & \textbf{977.92}\\ 
    \cline{3-11}
    &  & max        & 1868.37 & 1544.89 & 2110.61 & 1566.99 & 1656.06 & 1342.44 & 1202.91 & \textbf{1294.23} \\  
    \cline{3-11}
    & & \textbf{max/avg} & 1.70 & 1.86 & 2.23 & 1.69 & 1.70 & 1.30 & 1.26 & \textbf{1.32} \\
    \hline
    \multirow{7}{*}{Partial}  & \textbf{Energy} (US\$)                     & avg        & 283551 & 456028 & 516966 & 324368 & 314254 & 450992 & 342498 & \textbf{359510} \\ 
    \cline{2-11}
    & \multirow{3}{*}{\textbf{Water} ($\text{m}^3$)} & avg      &  14312.2 & 13249.7 & 11755.2 & 13903.6 & 13298.1 & 13584.9 & 13508.3 & \textbf{13619.4} \\ 
    \cline{3-11}
    &  & max   & 23961.6 & 25347.0 & 18852.3 & 24833.9 & 21734.6 & 19662.3 & 16824.7 & \textbf{18267.1}\\
    \cline{3-11}
    & & \textbf{max/avg} & 1.67 & 1.91 & 1.60 & 1.79 & 1.63 & 1.45 & 1.25 & \textbf{1.34}  \\
    \cline{2-11}
    & \multirow{3}{*}{\textbf{Carbon} (ton)}   & avg  & 1093.61 & 850.24 & 981.20 & 942.91 & 986.54 & 1035.97 & 964.03 & \textbf{980.46} \\ 
    \cline{3-11}
    &  & max  & 1874.35 & 1577.51 & 2110.77 & 1545.47 & 1695.69 & 1342.44 & 1240.80 & \textbf{1301.70}\\  
    \cline{3-11}
    & & \textbf{max/avg} &  1.71 & 1.86 & 2.15 & 1.64 & 1.72 & 1.30 & 1.29 & \textbf{1.33} \\
    \bottomrule
    \end{tabular}
\end{table*}

 As shown in Theorem~\ref{thm:cost_bound}, $\eta$ is an important parameter that determines the cost gap between \ouralg and its offline version \ouralgoffline. 
Thus, we also evaluate the impact of the learning rate $\eta$ in Fig.~\ref{fig:comparison_full}.  
The total cost is calculated by summing up the energy cost and 
equity-related carbon/water costs (weighted by $\mu_c$
and $\mu_w$, respectively).
Then, in line with  Section~\ref{sec:experiment_metrics},
we divide the total cost by 10 data center locations and
show it in Fig.~\ref{fig:full_total_cost}. 
As we increase the learning rate $\eta$, the total cost first decreases and then increases as suggested by Theorem~\ref{thm:cost_bound}.
Empirically, the optimal learning rate $\eta$ is around $3\times 10^{-5}$ 
in our setting. In Fig.~\ref{fig:full_avg_price}, the average energy 
cost increases as we increase $\eta$, since larger $\eta$ leads 
to more aggressive updates of the Lagrange multiplier $\kappa$. 
The Lagrange multiplier $\kappa$ is 
used to penalize the cost objective according to the environmental footprint, which means larger $\kappa$ shifts the objective function more towards environmental equity, as compared with purely minimizing the energy cost. Similarly, as we increase the learning rate $\eta$, the carbon and water footprints for the most disadvantaged regions decrease, as shown in Fig~\ref{fig:full_max_water} and~\ref{fig:full_max_carbon}. The underlying reason is similar --- a larger learning rate $\eta$ updates $\kappa$ more rapidly, eventually leading to more attention to equity-related costs. Interestingly, the average carbon and water footprints of \ouralg are very close to the offline version, \ouralgoffline. Like in the task of machine learning training, the learning rate hyperparemeter
can be tuned based on a validation dataset to get the desired performance
in practice.

\subsubsection{GLB with Partial Flexibility}
Now, we consider the partial-flexibility scenario in
which intra-continental workload routing is fully
flexible but inter-continental workload routing is partially restricted.
Specifically, 
we only allow partial inter-continental workload routing 
 as follows: workloads can be flexibly routed
 between Asia and the western U.S. (Nevada),
 and between Europe and the eastern U.S. (Virginia and Georgia). 

Our results are similar to those in the full-flexibility scenario.
Specifically, while the inter-continental workload routing restriction limits the GLB decision space,
\ouralgoffline still has the lowest carbon and water
footprints for the most disadvantaged regions.
Meanwhile, the average energy cost, carbon footprint, and water footprint
of \ouralgoffline are all comparable to or even lower than 
the other GLB algorithms. Thus, even without full flexibility, \ouralgoffline demonstrates a great potential
to address AI's environmental inequity in today's
geographically distributed data center infrastructures. 
Additionally, as shown in Table~\ref{table:comparison_all_results}, the performance of \ouralg is very close to its offline counterpart, \ouralgoffline.

\distancealg does not route workloads across data centers and
hence is not affected by the partial GLB flexibility.
Interestingly, the result of \carbonalg is not affected
by  the inter-continental workload routing restriction
in our setting, because the workloads
from each continent can be processed
in at least one low-carbon
data center in our setup (see Table~\ref{table:datacenter_details}).

With partial GLB flexibility, the impact of the learning rate
$\kappa$ is similar as that with full GLB flexibility. More details
about the empirical results  can be found in Appendix~\ref{sec:app_glb_partial}.

\subsubsection{GLB with Different $\mu_c$ and $\mu_w$}

Adjusting the weight hyperparameters $\mu_c$ and $\mu_w$ allows us to control the relative importance of the energy cost and environmental equity.  
Here,  by using the default setup, 
we evaluate how the weights for carbon and water footprints
impact the performance of different GLB algorithms.
We show the results in in Fig.~\ref{fig:compare_weight}. We only compare the performance of \costcarbonalg, \allalg, \ouralgoffline and \ouralg, as the other GLB algorithms are not affected by these weight hyperparameters.
Naturally,
by assigning lower weights to carbon and/or footprints, the emphasis on reducing the environmental inequity in terms of these footprints is lessened, allowing all the four GLB algorithms to have a reduced energy cost. However, this also
results in a higher maximum carbon and/or water footprint (as well as higher maximum-to-average ratios) for these algorithms, with the only exception being \costcarbonalg. 
More specifically, unlike the other algorithms, \costcarbonalg minimizes
the weighted sum of the energy cost and carbon footprint without
accounting for the impact of water footprint. As a result, it has a higher (maximum) water footprint as we increase the weight of carbon for reducing the carbon footprint. This empirical finding suggests that the goals of reducing carbon emission and water usage may not be aligned, or even in opposition, necessitating a joint optimization of their combined weighted sum. 

While \ouralg applies with any  $\mu_c\geq0$ and $\mu_w\geq0$, it is up to the AI system operator to tune weight hyperparameters
(e.g., based on  validation dataset) to achieve a desired outcome.
This is also a common and standard practice in real systems (see Google's dynamic capacity planning to balance the energy cost and environmental impacts \cite{Google_CarbonAwareComputing_PowerSystems_2023_9770383}).

\subsubsection{Evaluation over a Longer Timescale}\label{sec:experiment_longer_timescale}

The open-source BLOOM inference trace is only for 18 days \cite{ML_Carbon_Bloom_176B_Sasha_Luccioni_arXiv_2022_luccioni2022estimating} and used in our default setup. Due to limited availability
of public data, we extend the 18-day BLOOM inference trace
to 180 days by using data augmentation techniques to evaluate
the impacts of \ouralg in terms of environmental equity over a longer timescale. More specifically, we add 25\% random perturbations
and append the perturbed workload trace to the original one
to construct a 180-day trace. 
The results are shown in Table~\ref{table:long_sequence},
offering similar insights as in the default case for both full and partial GLB scenarios.
We can see that compared to the other equity-oblivious GLB algorithms, \ouralg can effectively reduce
the environmental inequity among different regions in terms
of the maximum-to-average ratio for both carbon and water footprints.
Even compared to \ouralgoffline, \ouralg delivers a similar performance
in terms of environmental equity while only marginally increasing
the total energy cost, which demonstrates the potential of \ouralg 
to address AI's emerging environmental inequity in practice without knowing all the future information.

\section{Related Work}

Our work is the first to address the critical concern
of AI's emerging environmental inequity by leveraging
GLB, and contributes to the GLB literature for cloud computing and
data centers
\cite{Carbon_ChasingCarbon_DataCenter_CaroleWu_IEEE_Micro_2022_10.1109/MM.2022.3163226,Carbon_CarbonExplorer_HolisticGreenDataCenter_CaroleWu_BenjaminLee_ASPLOS_2023_10.1145/3575693.3575754,Google_CarbonAwareComputing_PowerSystems_2023_9770383,RenWangUrgaonkarSivsubramaniam_Carbon_Mascots,UMass_CloudSpotOnlineServices_HPDC_2015_He:2015:CCH:2749246.2749275,Carbon_SustainableClouds_VirtualizingEnergy_DavidIrwin_AdamWierman_SoCC_2021_10.1145/3472883.3487009,UMass_CostRenewable_CDN_ICAC_2015_lee2015cost,Minghua_eEnergy_2014_Camacho:2014:BYB:2602044.2602068,WeiZhangOSU_SUSCOM_14,RaoLiuXieLiu_2010,Liu:2011:GGL:1993744.1993767,Gao:2012:EG:2377677.2377719,DataCenter_CuttingElectricBill_MIT_Sigcomm_2009_10.1145/1592568.1592584,Shaolei_Water_SpatioTemporal_GLB_TCC_2018_7420641,Le:2011:REC:2063384.2063413}. Specifically, prior studies focus
on reducing the total energy cost, carbon footprint,
  water footprint, or a weighted combination of these metrics;
 ignoring the potential for regional
 disparities. We show in this paper that existing GLB algorithms can potentially amplify environmental inequity by further exploiting already vulnerable regions. For example, GLB algorithms that aggressively
exploit lower electricity prices \cite{RaoLiuXieLiu_2010,DataCenter_CuttingElectricBill_MIT_Sigcomm_2009_10.1145/1592568.1592584} and/or more renewables \cite{Liu:2011:GGL:1993744.1993767,Gao:2012:EG:2377677.2377719}
may schedule more workloads to data centers (located in, for example, Arizona) that are extremely water-stressed; thus adding a disproportionately high pressure to local water systems.

Sustainable AI has received a significant amount of attention in recent years
\cite{ML_GPT3_Energy_Others_NIPS_2020_NEURIPS2020_1457c0d6,ML_LaMDA_Lanugage_Google_arXiv_2022_thoppilan2022lamda,ML_CarbonFoorptint_Google_JeffDean_Journal_2022_9810097,ML_Carbon_LargeModelTraining_Google_arXiv_2021_patterson2021carbon,GreenAI_EnergyPolicy_NLP_UMass_ACL_2019_strubell-etal-2019-energy,GreenAI_Washington_ACM_2020_10.1145/3381831,GreenAI_ReportingEnergyCarboon_Stanford_JMLR_2020_10.5555/3455716.3455964}.
To make AI more energy-efficient
and sustainable, a variety of approaches have been explored and studied, 
including
 computationally efficient
training and inference \cite{DNN_DeepSpeed_YuxiongHe_MSR_ICML_2022_rajbhandari2022deepspeed,ML_FrugalGPT_Cost_JamesZou_Stanford_arXiv_2023_chen2023frugalgpt}, 
energy-efficient GPU and accelerator designs \cite{DNN_AutoDNNChip_FPGA_ASIC_YingyanLin_Rice_DemingChen_UIUC_FPGA_2020_10.1145/3373087.3375306,DNN_NAS_AcceleratorAware_AutoML_SimulationPredictor_gupta2020accelerator,ML_CarbonFoorptint_Google_JeffDean_Journal_2022_9810097}, 
carbon-aware task scheduling \cite{GreenAI_ReportingEnergyCarboon_Stanford_JMLR_2020_10.5555/3455716.3455964,Carbon_SustainbleAI_CaroleWu_MLSys_2022_wu2022sustainable}, 
green cloud infrastructures \cite{ChaoLi_iSwitch_ISCA_2012_Li:2012:ICO:2337159.2337218,Carbon_MetricsSustainability_AnshulGandhi_StonyBrook_HotCarbon_2022,Carbon_SustainableClouds_VirtualizingEnergy_DavidIrwin_AdamWierman_SoCC_2021_10.1145/3472883.3487009,Carbon_CarbonExplorer_HolisticGreenDataCenter_CaroleWu_BenjaminLee_ASPLOS_2023_10.1145/3575693.3575754}, 
among others. 
While they are useful for overall sustainability, 
these studies do not address the emerging environmental equity 
among different regions for deploying AI services.
 Additionally, they have mostly focused
on carbon footprint, neglecting other crucial environmental
footprints, e.g., water footprint \cite{Google_SustainabilityReport_2022,Facebook_SustainabilityReport_2021,Microsoft_Water_Cloud_Sustainability_2023,Shaolei_Water_AI_Thirsty_arXiv_2023_li2023making}.
In contrast, we holistically consider both
carbon and water footprints and make novel contributions
to  sustainable AI from the perspective
of environmental equity.

There also exist non-\emph{computational} approaches to improving
AI's environmental sustainability. For example, data center operators have increasingly
adopted carbon-free energy such as solar and wind power
to lower AI's
carbon emissions \cite{Carbon_SustainbleAI_CaroleWu_MLSys_2022_wu2022sustainable,Google_SustainabilityReport_2022,Facebook_SustainabilityReport_2021,Microsoft_Water_Cloud_Sustainability_2023}. To cut
on-site potable water consumption and mitigate
the stress on already-limited freshwater resources, climate-conscious cooling
system designs (e.g.,  air-side economizers
and purifying non-potable water) have recently seen an uptick in the data center industry \cite{Google_Water_Commitments_website,Facebook_Water_2023_meta}.
These non-computational approaches alone
are typically not the most effective solution
to sustainable AI, and must be designed in conjunction with 
computational approaches (e.g., workload scheduling) 
\cite{Microsoft_CarbonAwareComputing_Treehouse_HotCarbon_202255728,Carbon_CarbonExplorer_HolisticGreenDataCenter_CaroleWu_BenjaminLee_ASPLOS_2023_10.1145/3575693.3575754,Microsoft_CarbonAwareSoftware_Whitepaper_2023,Google_CarbonAwareComputing_PowerSystems_2023_9770383}.
As such, our study of equity-aware GLB 
can  inform the planning of on-site carbon-free energy and 
cooling system renovation projects
to better achieve social and environmental justice.

Equity and fairness are crucial considerations for AI. 
The existing research in this space has predominantly
focused on mitigating prediction unfairness against disadvantaged
individuals and/or groups under a variety of settings \cite{Fair_ML_Survey_ACM_2022_10.1145/3494672,Fair_qFed_FederatedLearning_CMU_ICLR_2020_Li2020Fair,Fair_Representatoin_Toronto_ICML_2013_pmlr-v28-zemel13,Fair_BiasComputerSystems_ACM_Trans_1996_10.1145/230538.230561,Fair_MinimaxPareto_Fairness_ICML_2020_10.5555/3524938.3525565,Fair_Discrimination_DataMining_KDD_2008_10.1145/1401890.1401959,Fair_ConflictIndividual_Group_FAccT_2020_10.1145/3351095.3372864,Fair_SequentialDecision_Survey_MingyanLiu_UMich_Handbook_2021_Zhang2021,Fair_FairnessThroughAwareness_EarlyWork_ITCS_2012_10.1145/2090236.2090255,Fair_TestingPublicThreeDefinitions_LoanAllocation_Journal_YangLiu_UCSC_2020_SAXENA2020103238,Fair_FastFair_BiasMitigation_GNN_YanningShen_UCI_TMLR_2023_kose2023fastfair,Fair_MiniMaxGroupFairness_AaronRoth_Amazon_2021_Diana2021_10.1145/3461702.3462523}.
Our work on environmental equity adds a unique dimension of fairness
and greatly complements the existing rich body of research,
collaboratively and holistically building equitable and socially-responsible AI.

\section{Concluding Remarks}

In this paper, we take a first step to address 
the emerging environmental inequity of AI by 
balancing its regional negative environmental impact in an equitable
manner. 
Concretely,
we focus on the carbon and water footprints of AI model 
inference and propose equity-aware 
 GLB
to explicitly address the  environmental impact on the most disadvantaged region. \ouralg can optimize GLB decisions to fairly balance
AI's environmental cost across different regions in an online manner.
%We also 
%consider more advanced settings where there is on-site carbon-free energy
%available to power the AI workloads and where
%we can further exploit AI's energy-accuracy flexibility 
 %by dynamically choosing one or more AI models to serve the workloads.
We run trace-based simulations by considering a set of 10 geographically distributed data centers that serve inference requests for a large language AI model. The results highlight
that, compared to the existing GLB approaches,
our proposed equity-aware GLB can significantly  reduce
the regional disparity in terms of AI's carbon and water footprints. 

Our work demonstrates the need and great potential of equity-aware GLB
to address AI's emerging environmental equity. 
%An important future research
%problem is how to design an efficient online GLB algorithm
%to realize the potential in practice. This is a challenging
%problem that has not been well studied by the prior literature
%on GLB or online
%optimization. The key technical challenge is that reducing
%AI's long-term environmental impact on the most disadvantaged region
%(i.e., minimax in \eqref{eqn:objective}) requires all
%the future information, such as future AI workloads,
%which is only revealed sequentially in practice.
%Additionally, 
%our work also 
It opens up multiple new research directions to further
improve AI's environmental equity, such
as how to jointly optimize GLB and non-IT resource
(e.g., batteries) management
and how to leverage environmental science tools
to quantify
the impact of AI's carbon and water footprints on each region's
ecological system.

\section*{Acknowledgement}

Pengfei Li, Jianyi Yang, and Shaolei Ren were supported in part by the NSF under grants
CNS-1910208 and CCF-2324941. Adam Wierman
was supported in part by the NSF under
grants CNS-2146814, CPS-2136197, CNS-2106403 and NGSDI-2105648,
and by the Resnick Sustainability Institute.

{
      \bibliographystyle{unsrt}
        \bibliography{main_eenergy_2024.bbl}
}

\appendix

\begin{table*}[!ht]
\centering
\footnotesize
\caption{The detailed information of our data center locations. The 
values shown in the table are averaged over the 18-day period between September 23 and October 11, 2022. } 
\label{table:datacenter_details}
\begin{tabular}{c|c|c|c|c|c} 
\toprule
\multirow{2}{*}{\textbf{Country}}  & 
{\multirow{2}{*}{\textbf{State/Province}}}& 
{\multirow{2}{*}{\textbf{City}}}  & 
\textbf{Total WUE}  & 
\textbf{Carbon Intensity} & 
\textbf{Energy Price}  \\ 
& & & ($\text{m}^3$/MWh) & (ton/MWh)& (\$/MWh)\\
\hline
U.S. & Texas   & Midlothian  & 5.7397   & 0.4011 & 64.931 \\ 
\hline
U.S. & Virginia& Loudoun  & 5.9755  & 0.3741 & 77.793 \\ 
\hline
U.S. & Georgia & Douglas  & 5.9001  & 0.4188 & 80.566 \\ 
\hline
U.S. & Nevada  & Storey  & 4.9306   & 0.2980 & 84.738 \\ 
\hline
Germany  & Hessen  & Frankfurt & 4.5889   & 0.3295 & 315.233\\ 
\hline
Belgium  & Hainaut & Saint-Ghislain & 4.9316   & 0.4802 & 247.083\\ 
\hline
Netherlands & Groningen  & Eemshaven& 3.0928   & 0.4454 & 248.258\\ 
\hline
Denmark  & Fredericia & Fredericia  & 3.8900   & 0.1391 & 213.773\\ 
\hline
Japan & Chiba Prefecture & Inzai & 2.4989  & 0.3280 & 129.269\\ 
\hline
Singapore& Singapore  & Jurong West & 5.8652   & 0.5260 & 155.462\\
\bottomrule
\end{tabular}
\end{table*}
\begin{figure*}[!ht]
\centering
\subfloat{
\includegraphics[width=0.8\linewidth]{figures/legend_caption.png}
\label{fig:legend_1}
} \par%
\setcounter{subfigure}{0}
\subfloat[S][Total weighted cost]{
\includegraphics[width=0.32\linewidth]{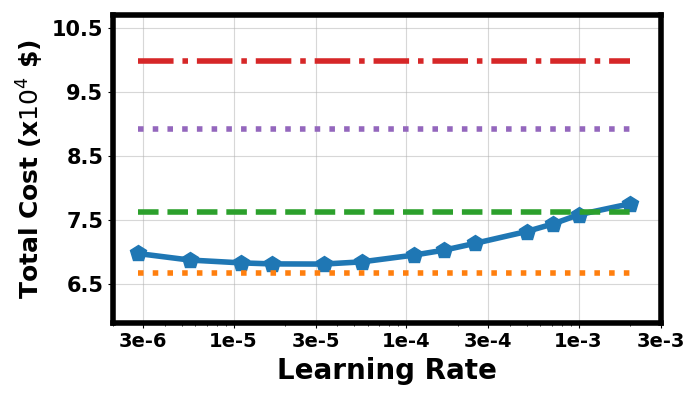}

}
\subfloat[S][Average energy cost]{
\includegraphics[width=0.32\linewidth]{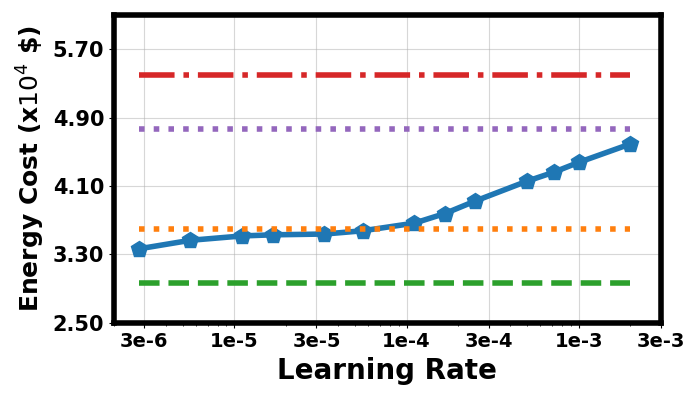}

}%
\subfloat[S][Maximum water]{%
\includegraphics[width=0.32\linewidth]{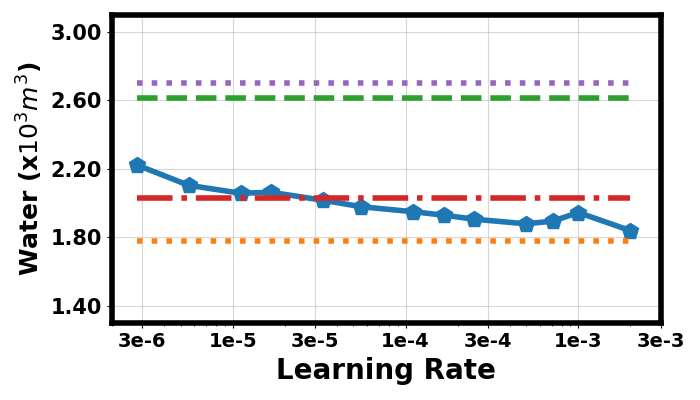}

}%
\par%
\subfloat[][Average water]{%
\includegraphics[width=0.32\linewidth]{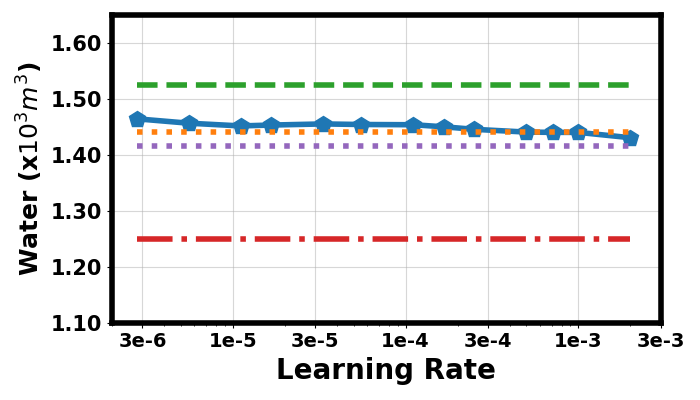}

}%
\subfloat[][Maximum carbon]{%
\includegraphics[width=0.32\linewidth]{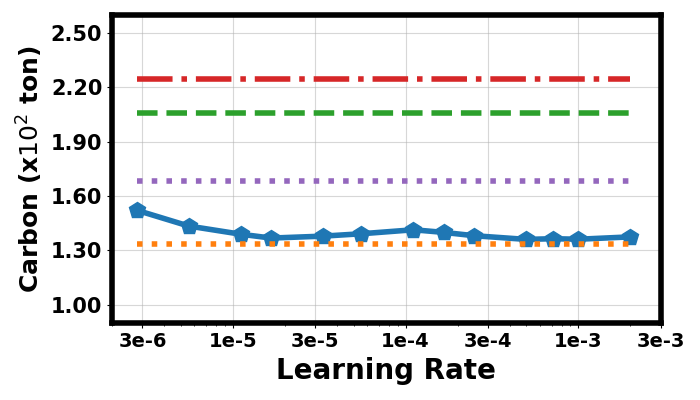}

}%
\subfloat[][Average carbon]{%
\includegraphics[width=0.32\linewidth]{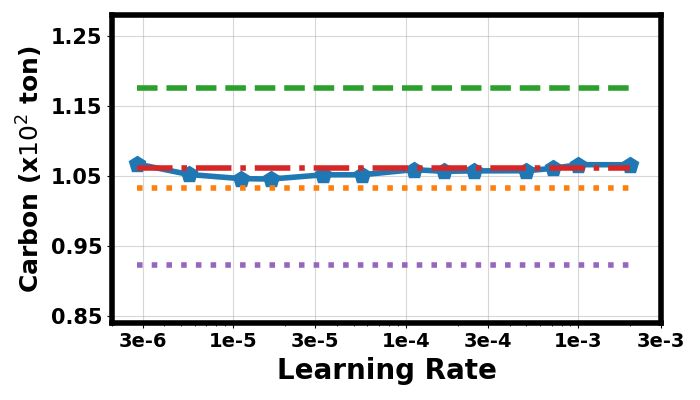}

}%
\par%
\caption{The energy cost, carbon and water footprint of \ouralg with different learning rates $\eta$ under partial GLB flexibility. The
results for \ouralgoffline, \costalg, \carbonalg,
and \wateralg are shown for comparison.}\label{fig:partial_GLB_results}
\end{figure*}

\section*{Appendix}

\section{Additional Details about the Experimental Setup}

In this section, we provide additional details about the data used in the experiments,
particularly the sources of our data on the fuel mix and water intensity factors.

We do not have free access to the hourly energy fuel mix
data for our data center locations in Europe and Asia.
Thus, we generate synthetic hourly fuel mixes for these
locations based on the U.S. data. 
We first obtain from \cite{Energy_Price_Fuel_Data_International_IEA_website} the average percentages of renewable and non-renewable energy in electricity generation between September 23 and October 11, 2022,
for each data center location in Europe and
Asia. Then, we scale the hourly energy fuel mix data in the U.S. to match the average percentages by mapping Texas' fuel mixes between June 1 and June 19, 2022, to Germany with non-renewable energy fuels 
scaled by 0.8503, Georgia's fuel mixes between June 1 and June 19, 2022, to Belgium with non-renewable energy fuels scaled by 1.5319, Georgia's fuel mixes between March 1 and March 19, 2022, to the Netherlands with non-renewable energy fuels
scaled by 1.2759, Oregon's fuel mixes between July 1 and July 19, 2022, to Denmark with non-renewable energy scaled by 0.2657, Nevada's fuel mixes between March 1 and March 19, 2022, to Japan with non-renewable energy fuels scaled by 3.2374, Georgia's fuel mixes between May 1 and May 19, 2022, to Singapore with non-renewable energy fuels scaled by 4.4875. We choose different 18-day periods
in order to de-correlate the European and Asian energy fuel
mix traces
from our actual U.S. data over the workload trace period  (between September 23
and October 11, 2022).

We also show the estimated energy water intensity factor (EWIF) in $\text{m}^3$/MWh for common energy fuel types
in the U.S. in Table~\ref{tab:water_carbon} \cite{Shaolei_Water_AI_Thirsty_arXiv_2023_li2023making,Water_EWIF_macknick2011review}, and the details of our 10 data center locations in Table~\ref{table:datacenter_details}.

\begin{table*}[!t]

\centering \caption{Estimated EWIF for common energy fuel types in the U.S. \cite{Water_EWIF_macknick2011review}.}
\begin{tabular}{c|c|c|c|c|c|c|c}
\toprule
 Fuel Type & Coal & Nuclear & Natural Gas & Solar (PV) & Wind & Other & Hydro \\
\hline EWIF (L/kWh) & 1.7 & 2.3 & 1.1 & 0 & 0 & 1.8 & 68 (0, if excluded)\\
\bottomrule
\end{tabular}
 \centering
\label{tab:water_carbon}
\end{table*}

\section{Additional Results for GLB with Partial Flexibility}\label{sec:app_glb_partial}

In this experiment, the goal is to evaluate the performance of \ouralg under partial GLB flexibility.
Intra-continental workload routing is fully
flexible but inter-continental workload routing is partially restricted.
Specifically, 
we only allow partial inter-continental workload routing 
 as follows: workloads can be flexibly routed
 between Asia and the western U.S. (Nevada),
 and between Europe and the eastern U.S. (Virginia and Georgia). 
Similar to the full GLB flexibility scenario, the
increase of the learning rate $\eta$ helps improve the environmental equity at the expense of increasing the energy cost.
The results are shown in Fig.~\ref{fig:partial_GLB_results}.

\section{Extension to Heterogeneous AI Models}\label{sec:heterogeneous_AI_models}

\begin{table*}[!t]
    %\scriptsize
    \footnotesize
    \caption{Comparison of different GLB algorithms in terms of energy cost, carbon and water footprint with heterogeneous AI models. 
    The results of \ouralg with
    the learning rate $\eta=10^{-5}$ are bolded.}   \label{table:heter_AI_models}
    \centering
    \begin{tabular}{c|c|c|c|c|c|c|c|c|c|c} 
    \toprule
    \textbf{GLB} & \multicolumn{2}{c|}{\multirow{2}{*}{\textbf{Metric}}} & \multicolumn{7}{c}{\textbf{Algorithm}} \\ 
    \cline{4-11}
       \textbf{Flexibility}    & \multicolumn{2}{c|}{}                        &  \costalg       &  \carbonalg     & \wateralg      & \costcarbonalg    & \allalg      & \distancealg         & {\ouralgoffline} & \textbf{\ouralg}\\ 
    \hline
    \multirow{7}{*}{Full} & \textbf{Energy} (US\$)                     & avg        & 20656 & 56891 & 58345 & 29047 & 26071 & 47038 & 29279 & \textbf{35034}\\ 
    \cline{2-11}
    & \multirow{3}{*}{\textbf{Water} ($\text{m}^3$)} & avg      &  1549.8 & 1311.9 & 1075.8 & 1476.7 & 1367.2 & 1446.7 & 1416.1 & \textbf{1424.6} \\ 
    \cline{3-11}
    &  & max   & 4521.5 & 4537.4 & 2788.5 & 3786.7 & 2640.2 & 2090.9 & 1612.7 & \textbf{1830.9} \\
    \cline{3-11}
    & & \textbf{max/avg} & 2.92 & 3.46 & 2.59 & 2.56 & 1.93 & 1.45 & 1.14 & \textbf{1.29}\\
    \cline{2-11}
    & \multirow{3}{*}{\textbf{Carbon} (ton)}   & avg  & 121.37 & 67.21 & 100.17 & 87.73 & 98.63 & 111.80 & 96.90 & \textbf{101.77}\\ 
    \cline{3-11}
    &  & max        & 352.84 & 171.07 & 375.92 & 233.49 & 201.94 & 143.46 & 116.59 & \textbf{129.94} \\  
    \cline{3-11}
    & & \textbf{max/avg} & 2.91 & 2.55 & 3.75 & 2.66 & 2.05 & 1.28 & 1.20 & \textbf{1.28} \\
    \hline
    \multirow{7}{*}{Full} & \textbf{Energy} (US\$)  & avg    &  21684 & 49516 & 54478 & 27329 & 26094 & 47038 & 30011 & \textbf{36170}\\ 
    \cline{2-11}
    & \multirow{3}{*}{\textbf{Water} ($\text{m}^3$)} & avg      &  1561.8 & 1411.2 & 1106.9 & 1503.3 & 1385.9 & 1446.7 & 1442.1 & \textbf{1437.6}  \\ 
    \cline{3-11}
    &  & max   & 4343.4 & 3699.6 & 2758.4 & 4106.1 & 3045.9 & 2090.9 & 1782.4 & \textbf{1885.4} \\
    \cline{3-11}
    & & \textbf{max/avg} & 2.78 & 2.62 & 2.49 & 2.73 & 2.20 & 1.45 & 1.24 & \textbf{1.31} \\
    \cline{2-11}
    & \multirow{3}{*}{\textbf{Carbon} (ton)}   & avg  & 119.50 & 79.67 & 103.45 & 96.26 & 103.17 & 111.80 & 99.46 & \textbf{103.33} \\ 
    \cline{3-11}
    &  & max        &  340.16 & 183.15 & 351.76 & 254.43 & 230.12 & 143.46 & 131.22 & \textbf{135.26} \\  
    \cline{3-11}
    & & \textbf{max/avg} &  2.85 & 2.30 & 3.40 & 2.64 & 2.23 & 1.28 & 1.32 & \textbf{1.31} \\
    \bottomrule
    \end{tabular}
\end{table*}

For the same inference service,
a set of heterogeneous AI models with
distinct computing resource consumption and accuracy performance may
be available via model pruning and compression in practice \cite{DNN_Compression_SongHan_ICLR_2016},
offering flexible energy-accuracy tradeoffs.
For example,
there are eight
different GPT-3 model sizes, ranging from the smallest one with
125 million parameters to the largest one with 175 billion parameters \cite{ML_GPT3_Energy_Others_NIPS_2020_NEURIPS2020_1457c0d6}. Now,
we extend our problem formulation to this generalized setting.

Suppose that there are a set $\mathcal{L}=\{1,\cdots,L\}$ of
heterogeneous AI models for our considered inference service. 
For time $t$, we can dynamically choose to run one or more AI models
to serve the incoming workloads. This
is also  equivalent to distributing
the  workload $\sum_{j\in\mathcal{J}}x_{i,j}(t)$ to
$L$ heterogeneous AI models within data center $i$.
We denote by $y_{i,l}(t)\geq0$ as the amount
of workloads distributed to AI model $l$ in data center $i$.
Naturally, $y_{i,l}(t)=0$ means that the AI model $l$ is not chosen
in data center $i$ at time $t$.

When deployed in data center $i$, the energy consumption
and server resource usage
of AI model $l$ for processing  workloads $y_{i,l}(t)$
are denoted by $e_{i,l}(y_{i,l}(t))$ and
$r_{i,l}(y_{i,l}(t))$, respectively. Thus,
the total server energy consumption in data center $i$ becomes
$\tilde{e}_i(y(t))=\sum_{l\in\mathcal{L}}e_{i,l}(y_{i,l}(t))$,
where $y(t)=\{y_{i,l}(t)\,|\,i\in\mathcal{N},l\in\mathcal{L}\}$ represents the collection of decisions
for workload assignment to different AI models.
Similarly, with heterogeneous AI models,
we can re-define the carbon footprint
and water footprint as $\tilde{c}_{i,t}(y(t))$
and $\tilde{w}_{i,t}(y(t))$ by replacing $e_i(x(t))$
with $\tilde{e}_i(y(t))=\sum_{l\in\mathcal{L}}e_{i,l}(y_{i,l}(t))$
in \eqref{eqn:carbon_footprint} and \eqref{eqn:water_footprint},
respectively.

To optimally distribute workloads to AI models with different  energy-accuracy tradeoffs,
we need to consider the inference \emph{cost} associated with different  accuracies, since otherwise always choosing the smallest model
can always result in the lowest energy consumption.
Specifically, we refer to the cost as performance cost
and denote it by $s_{l}(y_{i,l}(t))$, whose dependency on
$y_{i,l}(t)$ can be explained by noting that the
performance cost is potentially more significant when more users use
the model (i.e., $y_{i,l}(t)$ is larger).
Next, by combining the energy cost and performance cost,
we consider a generalized \emph{operational} cost as
follows:
\begin{equation}\label{eqn:operational_cost_generalized}
  \tilde{g}_t({y}(t))=
  \sum_{i\in\mathcal{N}} \sum_{l\in\mathcal{L}}\left[p_{i,t}\gamma_{i,t}\cdot e_{i,l}(y_{i,l}(t))+\phi\cdot s_{l}(y_{i,l}(t))\right],
\end{equation}
where
 the hyperparameter $\phi\geq0$ converts
the performance cost $s_{l}(y_{i,l}(t))$ to a monetary value and indicates
the importance of inference performance relative to the energy cost.

Finally, we formulate the generalized GLB problem with heterogeneous AI models
 as follows:
\begin{subequations}\label{eqn:objective_offline_heterogeneous}
 \begin{gather}\label{eqn:objective_heterogeneous}
 \begin{gathered}
        \min_{x(t),y(t)} \sum_{t=1}^T\tilde{g}_t(y(t))+\mu_c\cdot
       \max_{i\in\mathcal{N}}\left[\mathcal{H}_{i,c}\left(\sum_{t=1}^T\tilde{c}_{i,t}(y(t))\right)\right]\\+\mu_w\cdot
\max_{i\in\mathcal{N}}\left[\mathcal{H}_{i,w}\left(\sum_{t=1}^T\tilde{w}_{i,t}\left(y(t)\right)\right)\right],
\end{gathered}
       \\
       \label{eqn:constraint_delay_assignment_heterogeneous}
      s.t. \;\;\;\;\;\;    x_{i,j}(t)=0, \;\text{ if } B_{i,j}=0,\;\;\;\; \forall\; i\in{\mathcal{N}},
       j\in{\mathcal{J}}, t=1,\cdots,T,
          \\
       \label{eqn:constraint_gateway_heterogeneous}
         \sum_{x\in\mathcal{N}}x_{i,j}(t)= \lambda_{j,t}, \;\;\;\; \forall
       \; j\in\mathcal{J}, t=1,\cdots,T,\\
        \label{eqn:constraint_datacenter_capacity_heterogeneous}
      \sum_{k\in\mathcal{K}}r_{i,k}\left(y_{i,k}(t)\right)\leq M_i, \;\;\;\; \forall
       \; i\in\mathcal{N},t=1,\cdots,T,\\
       \label{eqn:constraint_datacenter_AI_model_heterogeneous}
         \sum_{j\in\mathcal{J}}x_{i,j}(t)= \sum_{k\in\mathcal{K}}y_{i,k}(t), \;\;\;\; \forall
       \; j\in\mathcal{J}, t=1,\cdots,T,
 \end{gather}
\end{subequations}
where the objective
\eqref{eqn:objective_heterogeneous} is to minimize
the generalized operational cost (including both energy
and performance costs) while addressing environmental inequity,
 the constraint \eqref{eqn:constraint_datacenter_capacity_heterogeneous}
means that the total resource demand must be no more than the server cluster's capacity, and the
last constraint \eqref{eqn:constraint_datacenter_AI_model_heterogeneous}
ensures that the workload assigned to each data center
is always served by
 one of the heterogeneous AI models.
The problem \eqref{eqn:objective_heterogeneous}--\eqref{eqn:constraint_datacenter_AI_model_heterogeneous}
can be solved by introducing auxiliary variables and
optimizing the Lagrangian with estimated dual variables.
The algorithm is  similar to Algorithm~\ref{alg:dmd_with_equity}.

Next, we conduct a trace-based simulation to evaluate the performance of different GLB algorithms with heterogeneous AI models. We consider 
a similar setup as the default one with homogeneous AI models, where 
we keep the weights $\mu_w$ and $\mu_c$ unchanged. 
As the open-sourced BLOOM has only one model size and is not suitable
for the heterogeneous AI model case,
we consider the Llama-2 model released by Meta with three different available model sizes (7B, 13B, and 70B), corresponding to different accuracies and energy demands \cite{AI_Llama2_Meta_arXiv_2023_touvron2023llama}. 
 We normalize the average inference accuracy and energy consumption by that of the largest model. We set the accuracy performance weight such that
 the average inference accuracy is roughly the same as
 that of the model with the medium size. As each performance weight
 corresponds to an average inference constraint,
 this is essentially equivalent
 to constraining
 the average inference accuracy (weighted by the amount of requests
 for each model) to be equal to that of the medium-size model.
 Each of the 10 geo-distributed data centers can handle a specified quantity of requests for each model size subject to the total capacity constraint. By using the same traces
for carbon intensity, water usage efficiency, and workloads as in the homogeneous case, 
we run different GLB algorithms and show the results Table~\ref{table:heter_AI_models}.

The results provide similar insights as in the homogeneous case for both full and partial GLB scenarios. 
Specifically, we see that compared to the other equity-oblivious GLB algorithms, \ouralg can effectively reduce
the environmental inequity among different regions in terms
of the maximum-to-average ratio for both carbon and water footprints.
Additionally, even compared to \ouralgoffline, \ouralg delivers a comparable performance
in terms of environmental equity while only slightly increasing
the total energy cost. Again, this demonstrates the potential of \ouralg 
to address AI's emerging environmental inequity.

\section{Proof of Theorem \ref{thm:cost_bound}}

When the reference function is $h(a) = \frac{1}{2}\| a \|^2$, the update rule in Line 7 in Algorithm \ref{alg:dmd_with_equity} can be rewritten as 
\begin{equation}\label{eqn:queue_update}
    \frac{1}{\eta}\kappa_{t+1} = \left[\frac{1}{\eta} \kappa_t -  d_t \right]^+ =  \left[\frac{1}{\eta}\kappa_t +  ( \begin{bmatrix}  \cC_t(x(t))  \\   \cW_t(x(t))  \end{bmatrix} - \begin{bmatrix} z_c(t)\\z_w(t)\end{bmatrix}
    )\right]^+
\end{equation}
where $[x]^+= x$ when $x$ is positive, otherwise it's set as zero. Given the dual variable $\kappa_t$, the optimization goal of Line 4 and Line 5 can be written as 
\begin{equation}\label{eqn:single_step_lag}
\begin{aligned}
    \min_{x(t), z_c(t), z_w(t) }\frac{1}{\eta} \biggl(& g_t(x(t)) +   \mu_c \| \cH_{c}({z}_{c}(t)) \|_{\infty}  \\
    &+ \mu_w \| \cH_{w}({z}_{w}(t)) \|_{\infty}\biggr) + \frac{\kappa_t^\top}{\eta} \cdot (-d_t)
\end{aligned}
\end{equation}
by multiplying Eqn~\eqref{eqn:single_step_lag} with $\eta$, it corresponds to the change of Lagrange function at time $t$. Now we define a new variable $\Delta_1(t) = \frac{1}{2 \eta^2}(\kappa_{t+1}^2 -\kappa_{t}^2) $ to quantify the change of Lagrange multiplier $\kappa_t$.  Our next step is to provide bounds on this dual variable $\kappa_t$, which is done by the following lemma.

\begin{lemma}\label{lemma:queue_length}
    If the reference function $h(a) = \frac{1}{2} \| a\|^2$ and $\kappa_{1} = 0$, then the dual variable $\kappa_t$ is bounded by
    \begin{equation}
        \|\kappa_{T+1}\| \leq \eta \sqrt{2T(B + \frac{M\theta_m}{\eta}( \mu_c c_m + \mu_w w_m))}
    \end{equation}
    where constant $B = \frac{N}{2}(\bar{z}_c^2 + \bar{z}_w^2)$.
\end{lemma}
\begin{proof}
     From Eqn~\eqref{eqn:queue_update}, we have the following inequality
    \begin{align}
        \Delta_1(t) &\leq  \frac{1}{2\eta^2}\left((\kappa_t - \eta d_t)^2 -  \kappa_t^2\right)\\
        &= \frac{1}{\eta}\kappa_t\cdot (- d_t) + \frac{1}{2}d_t^2\\
        &\leq \frac{1}{\eta}\kappa_t\cdot (- d_t)  +  B
    \end{align}
    where $B = \frac{N}{2} \cdot (\bar{z}_{c}^2 + \bar{z}_{w}^2)$, the first inequality is based on $([x]^+)^2 \leq x^2$ and the second inequality comes from our assumption that $\bar{z}_{c}$ and $\bar{z}_{w}$ are the largest possible values of the carbon and water footprint. 
    
    Suppose at time $t$, the optimal solution for Eqn~\eqref{eqn:single_step_lag} is $x(t)^\dagger$,\\$ {z}_{c}(t)^\dagger, z_w(t)^\dagger$,
    for any other $ {z}_{c}(t)' \in \cZ_c, z_w(t)' \in \cZ_w $  we have
    \begin{equation}\label{eqn:single_step_ppd}
        \begin{aligned}
        &\Delta_1(t) + \frac{1}{\eta} \left(g_t(x(t)^\dagger) +   \mu_c\| \cH_{c}({z}_{c}(t)^\dagger) \|_{\infty} + \mu_w \| \cH_{w}({z}_{w}(t)^\dagger) \|_{\infty}\right)\\
        \leq & B+ \frac{1}{\eta} \left(g_t(x(t)^\dagger) +   \mu_c \| \cH_{c}(z_c(t)') \|_{\infty} + \mu_w \| \cH_{w}(z_w(t)') \|_{\infty}\right) \\
        &+ \frac{\kappa_t^\top}{\eta} \cdot (\begin{bmatrix}  \cC_t(x(t)^\dagger)  \\   \cW_t(x(t)^\dagger)  \end{bmatrix} - \begin{bmatrix} z_c(t)'\\z_w(t)'\end{bmatrix} )
    \end{aligned}
    \end{equation}
In the second inequality, we choose $z_c(t)'$ and $z_w(t)'$ such that the term $\begin{bmatrix}  \cC_t(x(t)^\dagger)  \\   \cW_t(x(t)^\dagger)  \end{bmatrix} - \begin{bmatrix} z_c(t)'\\z_w(t)'\end{bmatrix} = 0$ , then we have
    \begin{align}
         &\begin{aligned}
             \Delta_1(t) \leq &B+ \frac{\mu_c}{\eta} \left( \| \cH_{c}(z_c(t)') \|_{\infty} -  \| \cH_{c}({z}_{c}(t)^\dagger) \|_{\infty}\right) + \\
            & \frac{\mu_w}{\eta}\left( \| \cH_{w}(z_w(t)') \|_{\infty}  -\| \cH_{w}({z}_{w}(t)^\dagger) \|_{\infty}\right)
            \end{aligned} \\
        & \quad \quad \leq B + \frac{M\theta_m}{\eta}( \mu_c c_m + \mu_w w_m)
    \end{align}
where the second inequality results from the assumption of maximum carbon or water price, the maximum datacenter capacity $M$ and $\theta_m$, the maximum gradient of function $\cH_w(\cdot)$ and $\cH_c(\cdot)$.
By summing up $\Delta_1(t)$ through $t=1$ to $T$, then we have
\begin{equation}
    \frac{1}{2\eta^2}(\kappa_{T+1}^2 - \kappa_{1}^2) \leq T(B + \frac{M\theta_m}{\eta}( \mu_c c_m + \mu_w w_m))
\end{equation}
\end{proof}

Using the previous result, we can now proceed by proving the following technical lemma, which when combined with the analysis above will let us complete the proof.

\begin{lemma}
     Suppose the optimal solution for Eqn~\eqref{eqn:single_step_lag} is $x_{1:T}^\dagger$, $z_{c,1:T}^\dagger$ and $z_{w,1:T}^\dagger$,  for any ${x(t)', z_c(t)', z_w(t)'}$ satisfying the constraints in Eqn~\eqref{eqn:action_single_step_constraint} -- \eqref{eqn:constraint_gateway_dual_2} we have
    \begin{equation}\label{eqn:total_cost_ineq}
        \begin{aligned}
            &\sum_{t=1}^T \Delta_{1}(t) + \frac{1}{\eta} \sum_{t=1}^T \biggl[g_t(x(t)^\dagger) +  \mu_c \| \cH_{c}(z_c(t)^\dagger) \|_{\infty}  \\
            &  + \mu_w \| \cH_{w}(z_w(t)^\dagger) \|_{\infty} \biggr]\\
            \leq & B T + B T (T-1)  +  \frac{\kappa_1}{\eta} \sum_{t=1}^T \biggl( \begin{bmatrix}  \cC_t(x(t)'   \\ \cW_t(x(t)')  \end{bmatrix} - \begin{bmatrix}z_{c}(t)'\\ z_{w}(t)'\end{bmatrix} \biggr)\\
            &+ \frac{1}{\eta}\sum_{t=1}^T \bigl[g_t(x(t)') +  \mu_c \| \cH_{c}(z_c(t)') \|_{\infty} + \mu_w \| \cH_{w}(z_w(t)') \|_{\infty} \bigr]
        \end{aligned}
    \end{equation}
\end{lemma}
\begin{proof}
Similar to Eqn~\eqref{eqn:single_step_ppd}, for any other $x(t) \in \mathcal{X}_t$, $ {z}_{c}(t)' \in \cZ_c, z_w(t)' \in \cZ_w $  we have
    \begin{equation}\label{eqn:single_step_ppd_2}
        \begin{aligned}
        &\Delta_1(t) + \frac{1}{\eta} \left(g_t(x(t)^\dagger) +   \mu_c\| \cH_{c}({z}_{c}(t)^\dagger) \|_{\infty} + \mu_w \| \cH_{w}({z}_{w}(t)^\dagger) \|_{\infty}\right)\\
        \leq & B+ \frac{1}{\eta} \left(g_t(x(t)') +   \mu_c \| \cH_{c}(z_c(t)') \|_{\infty} + \mu_w \| \cH_{w}(z_w(t)') \|_{\infty}\right) \\
        &+ \frac{\kappa_t^\top}{\eta} \cdot (\begin{bmatrix}  \cC_t(x(t)')  \\   \cW_t(x(t)')  \end{bmatrix} - \begin{bmatrix} z_c(t)'\\z_w(t)'\end{bmatrix} )
    \end{aligned}
    \end{equation}  
Now we define the subgradient of the action $x(t)'$ as $d_t' = \begin{bmatrix}  \cC_t(x(t)')  \\   \cW_t(x(t)')  \end{bmatrix} - \begin{bmatrix} z_c(t)'\\z_w(t)'\end{bmatrix}$. According to the update rule of $\kappa_t$, the maximum difference between dual variables at $t=1$ and $t = \tau +1$ is bounded by
\begin{equation}
   -\tau N (\bar{z}_c^2 + \bar{z}_w^2) \leq  \langle \frac{\kappa_{\tau+1}}{\eta} - \frac{\kappa_1}{\eta}, d_t' \rangle \leq \tau N (\bar{z}_c^2 + \bar{z}_w^2)
\end{equation} 
Therefore, for all $t \in [1,T]$, we have
\begin{equation}
\begin{aligned}
    \frac{\kappa_t^\top}{\eta} \cdot (-d_t') &\leq  \frac{\kappa_1^\top}{\eta} \cdot (-d_t') + N(t-1) (\bar{z}_c^2 + \bar{z}_w^2)  
\end{aligned}
\end{equation}
By summing up the inequality, we have
\begin{equation}
\begin{aligned}
    \sum_{t=1}^{T}\frac{\kappa_t^\top}{\eta} \cdot (-d_t') &\leq \frac{\kappa_1^\top}{\eta} \left( \sum_{t=1}^T (-d_t')\right) + N (\bar{z}_c^2 + \bar{z}_w^2)  \sum_{t=1}^T t-1\\
    &= \frac{\kappa_1^\top}{\eta} \left( \sum_{t=1}^T (-d_t')\right) + {T(T-1)} \frac{N}{2} (\bar{z}_c^2 + \bar{z}_w^2)
\end{aligned}
\end{equation}
By summing up Eqn~\eqref{eqn:single_step_ppd_2} through $t=1$ to $T$, we finish the proof.
\end{proof}

We are now ready to complete the proof.  Note that  $\sum_{t=1}^T \Delta_{1}(t) \geq 0$. Suppose $x_{1:T}^*$ is the optimal solution to the Eqn~\eqref{eqn:objective_offline}, which also satisfies the constraints in Eqn~\eqref{eqn:action_single_step_constraint} -- \eqref{eqn:constraint_gateway_dual_2}. Substituting  $x_{1:T}^*$ back to Eqn \eqref{eqn:total_cost_ineq} gives
\begin{equation}\label{eqn:pure_objective_difference}
\begin{aligned}
    &\frac{1}{T} \sum_{\tau=t_0}^T \Bigl[g_t(x(t)^\dagger) +  \mu_c \| \cH_{c}(z_c(t)^\dagger) \|_{\infty} + \mu_w \| \cH_{w}(z_w(t)^\dagger) \|_{\infty}  \Bigr]\\
    \leq &\eta (B + B(T-1)) + \frac{1}{T}\sum_{\tau=t_0}^T \Bigl[g_t(x(t)^*) +  \mu_c \| \cH_{c}(z_c(t)^*) \|_{\infty} \\
    &+ \mu_w \| \cH_{w}(z_w(t)^*) \|_{\infty} \Bigr]
\end{aligned}
\end{equation}
where the $\sum_{t=1}^T \begin{bmatrix} z_c(t)^*\\z_w(t)^*\end{bmatrix}= \sum_{t=1}^T\begin{bmatrix}  \cC_t(x(t)^*)  \\   \cW_t(x(t)^*)  \end{bmatrix}$, so the term $\kappa_1 \sum_{t=1}^T d_t$ is equal to zero. 

The left hand side of Eqn~\eqref{eqn:pure_objective_difference} is mixed up with GLB decisions $x(t)^\dagger$ and auxiliary variables $z_{c}(t)^\dagger$, $z_{w}(t)^\dagger$. The next step is to eliminate these auxiliary variables by bounding their difference. Based on the update rule of $\kappa_t$ and Lemma \ref{lemma:queue_length}, we have
\begin{equation}
\begin{aligned}
     &\frac{1}{T}\| \sum_{t=1}^T \left(\begin{bmatrix}  \cC_t(x(t)')  \\ \cW_t(x(t)')  \end{bmatrix} - \begin{bmatrix}z_{c}(t)'\\ z_{w}(t)'\end{bmatrix}\right)\|\\
     \leq &\frac{1}{T} \left(\|\frac{\kappa_{T+1}}{\eta}\| - \| \frac{\kappa_1}{\eta} \| \right) \\  \leq&  \sqrt{\frac{2}{T}(B + \frac{M\theta_m}{\eta}( \mu_c c_m + \mu_w w_m))}
\end{aligned}
\end{equation}

The maximum gradient of $\mu_c \| \cH_{c}(z_c(t)^\dagger) \|_{\infty} + \mu_w \| \cH_{w}(z_w(t)^\dagger) \|_{\infty}$ with respect to $[z_c(t), z_w(t)]$, is always bounded by $C = \theta_m(\mu_c + \mu_w)$. For simplicity, we define $D = \theta_m(\mu_c c_m + \mu_w w_m)$, then we have

\begin{equation}\label{eqn:cost_difference_dual}
    \begin{aligned}
    &\mu_c \| \cH_{c}(\frac{1}{T} \sum_{\tau=t_0}^T \cC_t(x(t)^\dagger)) \|_{\infty} + \mu_w \| \cH_{w}(\frac{1}{T} \sum_{\tau=t_0}^T\cC_t(x(t)^\dagger)) \|_{\infty} \\
    \leq&  \mu_c \| \cH_{c}(\frac{1}{T} \sum_{\tau=t_0}^T z_c(t)^\dagger) \|_{\infty} + \mu_w \| \cH_{w}(\frac{1}{T} \sum_{\tau=t_0}^Tz_w(t)^\dagger) \|_{\infty} \\
    & + \Bigl[ \theta_m(\mu_c + \mu_w)\Bigr] \sqrt{\frac{2}{T}(B + \frac{M}{\eta}D)}\\
    \leq &\frac{1}{T} \sum_{\tau=t_0}^T \biggl[\mu_c \| \cH_{c}(z_c(t)^\dagger) \|_{\infty}  + \mu_w \| \cH_{w}(z_w(t)^\dagger) \|_{\infty}\biggr]\\
    &+ \Bigl[ \theta_m(\mu_c + \mu_w)\Bigr] \sqrt{\frac{2}{T}(B + \frac{M}{\eta}D)}
\end{aligned}
\end{equation}
where the first inequality is based on the maximum gradient $D$, the second inequality is from Jensen's inequality. By substituting Eqn~\eqref{eqn:cost_difference_dual} back to Eqn~\eqref{eqn:pure_objective_difference}, we recover the cost objective in Eqn~\eqref{eqn:objective} and finish the proof.

\end{document}